\documentclass{article}

\usepackage{arxiv}

\usepackage[utf8]{inputenc} 
\usepackage[T1]{fontenc}    
\usepackage{hyperref}       
\usepackage{url}            
\usepackage{booktabs}       
\usepackage{amsfonts}       
\usepackage{nicefrac}       
\usepackage{microtype}      
\usepackage{lipsum}		
\usepackage{graphicx}
\usepackage{natbib}
\usepackage{doi}

\usepackage{amsmath,amsfonts,bm}

\def\eqref#1{equation~\ref{#1}}

\def\1{\bm{1}}

\def\rvg{{\mathbf{g}}}

\def\rvs{{\mathbf{s}}}

\def\rvx{{\mathbf{x}}}
\def\rvy{{\mathbf{y}}}

\def\ervg{{\textnormal{g}}}

\def\ervs{{\textnormal{s}}}

\def\rmI{{\mathbf{I}}}

\DeclareMathAlphabet{\mathsfit}{\encodingdefault}{\sfdefault}{m}{sl}
\SetMathAlphabet{\mathsfit}{bold}{\encodingdefault}{\sfdefault}{bx}{n}

\DeclareMathOperator*{\argmin}{arg\,min}

\usepackage{cleveref}
\usepackage{subfigure}
\usepackage{arydshln}
\usepackage{xcolor}         
\usepackage{amsthm}
\usepackage{multicol}
\usepackage{multirow}
\usepackage{amsmath}
\usepackage{algorithm}
\usepackage{algorithmic}
\usepackage[flushleft]{threeparttable}
\crefname{ineq}{inequality}{inequalities}
\creflabelformat{ineq}{#2{\upshape(#1)}#3} 

\newtheorem{theorem}{Theorem}[section]

\newtheorem{lemma}[theorem]{Lemma}

\newtheorem{definition}[theorem]{Definition}
\newtheorem{assumption}[theorem]{Assumption}
\newtheorem{remark}[theorem]{Remark}

\title{Stochastic Two Points Method for Deep Model Zeroth-order Optimization}

\author{
  Yijiang Pang\\
  Michigan State University\\
   East Lansing, Michigan, USA 48823 \\
  \texttt{pangyiji@msu.edu} \\
   \And
   Jiayu Zhou \\
   Michigan State University \\
   East Lansing, Michigan, USA 48823 \\
   \texttt{jiayuz@msu.edu} \\
}

\hypersetup{
pdftitle={A template for the arxiv style},
pdfsubject={q-bio.NC, q-bio.QM},
pdfauthor={David S.~Hippocampus, Elias D.~Striatum},
pdfkeywords={First keyword, Second keyword, More},
}

\begin{document}
\maketitle

\begin{abstract}
Large foundation models, such as large language models, have performed exceptionally well in various application scenarios. 
Building or fully fine-tuning such large models is usually prohibitive due to either hardware budget or lack of access to backpropagation.
The zeroth-order methods offer a promising direction for tackling this challenge, where only forward passes are needed to update the model.   
This paper introduces an efficient Stochastic Two-Point (S2P) approach within the gradient-free regime.  
We present the theoretical convergence properties of S2P under the general and relaxed smoothness assumptions, and the derived results help understand and inherently connect the two popular types of zeroth-order methods, basic random search and stochastic three-point method.
The theoretical properties also shed light on a Variant of S2P (VS2P), through exploiting our new convergence properties that better represent the dynamics of deep models in training.
Our comprehensive empirical results show that VS2P is highly effective in optimizing objectives for deep models. It outperforms or achieves competitive performance compared to standard methods across various model types and scales.
\end{abstract}

\section{Introduction}
\label{sec_introduction}

Utilizing pre-trained large models for various downstream tasks has emerged as a prominent trend, particularly in the context of Large Language Models (LLMs), which demand substantial computational resources and data during their initial training phase~\cite{devlin2018bert, bommasani2021opportunities}. 
Different from smaller deep models, full fine-tuning these models is often prohibitive due to the massive computing resources needed. 
Therefore, techniques such as parameter-efficient tuning, including prompt tuning~\cite{lester2021power} and LoRA~\cite{hu2021lora}, as well as zeroth-order methods~\cite{malladi2023fine, prasad2022grips}, are developed and demonstrated satisfactory performance.
Among these approaches, zeroth-order methods have become especially attractive recently since they only rely on function values, often referred to as zeroth-order information, to optimize models, avoid memory-intensive back-propagation, and enable non-differentiable optimization and full or partial fine-tuning with minimum computing resources. 
For instance, MeZO, a recently proposed zeroth-order method, exhibits memory savings up to 12 times compared with standard full fine-tuning and shows competitive performance with full fine-tuning\&parameter-efficient tuning methods~\cite{malladi2023fine}.

Zeroth-order optimization is broadly investigated and generally analyzed within the framework of optimizing the non-convex optimization problem $\min_{\rvx\in\mathbb{R}^{d}} f(\rvx)$, where the $f:\mathbb{R}^{d}\rightarrow \mathbb{R}$, and the derivatives are not directly accessible.
The complexity of this problem is studied over function query complexity, namely the complexity in terms of the number of function evaluations.
And, existing analyses of zeroth-order approaches mainly focus on convergence to $\epsilon$-first-order stationary points under the $L$-smoothness assumption~\cite{nesterov2017random, bergou2020stochastic}.

Zeroth-order optimization can be categorized into two types by whether or not it explicitly approximates gradient: gradient estimator and direct search~\cite{ghadimi2013stochastic, chen2020accelerated, lewis2000direct, conn2009introduction}.
\emph{Gradient estimator} methods compute an estimate of the gradient through zeroth-order information to optimize $f$, i.e., random (gradient-free) oracles.
%
Gaussian smoothing is a gradient estimator algorithm that initially uses random-directions stochastic approximation as a random oracle, and their work establishes the framework of analyzing the convergence properties of $f$ once explicitly obtaining mean squared error between the approximated gradient and true gradient~\cite{nesterov2017random}. 
On the other hand, the \emph{direct search} generally optimizes $f$ by updating the objective function along fixed or randomized directions with fixed or adaptive step size (e.g., reduce step size when the selected directions get rejected)~\cite{vicente2013worst}.
Stochastic Three Points (STP)~\cite{bergou2020stochastic} is a representative approach in this category. 
The intuition behind STP is straightforward, with the condition $\mathbb{E}|\rvs^{T}\nabla f(\rvx)| \geq C||\nabla f(\rvx)||$ where $\rvs$ is a random vector sampled from specific distributions and $C$ is a small positive constant, one of the STP directions $\pm\rvs$ with an appropriate step size consistently decreases the objective function in expectation.

In practice, it is often useful to sample a symmetric (two-sided) random perturbation per update for optimization problems. 
This approach finds practical utility in scenarios like LLM fine-tuning~\cite{malladi2023fine, zelikman2023just} and provides theoretical enhancement when exploiting multiple random perturbations per update~\cite{salimans2017evolution, mania2018simple}.
Examples include STP from direct search and two-sided \textit{G}radient \textit{A}pproximation Gaussian smoothing (basic random search), abbreviated as GA.
When using symmetric perturbations, their respective updates are given by: 
\begin{align*}
\text{STP: } & \rvx_{k+1} = \argmin \{f(\rvx_{k} + \alpha\rvs_{k}), f(\rvx_{k} - \alpha\rvs_{k}), f(\rvx_{k})\},\\
\text{GA: }& \rvx_{k+1} = \rvx_{k} -\alpha \tfrac{f(\rvx_{k} + \rho\rvs_{k}) - f(\rvx_{k} - \rho\rvs_{k})}{2\rho}\rvs_{k},
\end{align*}
where $\rvs_{k}\sim\mathcal{N}(0, \rmI)$, $\alpha$ denotes step size, and $\rho$ denotes smoothing parameter. 
Notably, GA is principally equivalent to MeZO as presented in~\cite{malladi2023fine}, which reduces memory consumption compared with GA through implementation trick by doing the twice forward passes sequentially instead of in parallel.

The convergence of both approaches relies on the (general) $L$-smoothness assumption, a widely employed concept in non-convex optimization but limits the analysis to functions bounded by quadratic behavior; for instance, a simple function like $x^{4}$ is not globally $L$-smooth. Recently,~\cite{zhang2019gradient} proposed the (relaxed) $L_{0}, L_{1}$-smoothness assumption and showed that the relaxed smoothness is a more realistic assumption for many real-world tasks, especially for deep models. And, the relaxed smoothness assumption has a larger compatibility range compared with the general smoothness assumption, e.g., compatible with univariate polynomial and exponential functions~\cite{zhang2019gradient, zhang2020improved, danilova2022recent}.

Given these insights, there is a growing trend to adapt theories of existing popular techniques to the relaxed smoothness assumption~\cite{wang2023convergence, li2023convergence}, gaining insights into the development of new algorithms.
For instance, it has been shown that vanilla SGD can be arbitrarily slower compared to clipped SGD and Adam~\cite{kingma2014adam} under relaxed smoothness assumption~\cite{zhang2019gradient,wang2022provable}. This partially supports the superiority of Adam-like methods over SGD and inspires further improvements~\cite{li2024convergence}.

In this paper, distinct from existing work, we build upon relaxed smoothness and advance the efficiency of the zeroth order optimization by proposing a new approach called Stochastic Two-Point (S2P).
 %
The paper has the following contributions to zeroth-order methods for large deep models:
\begin{itemize}
    \item We analyze the convergence properties of the proposed S2P under general and relaxed smoothness assumptions. The basic form of S2P has query complexity $\mathcal{O}(\frac{d}{\epsilon^{2}})$ under general smoothness assumption, which is the same with~\cite{nesterov2017random, bergou2020stochastic}. 
    To our knowledge, the analysis of query complexity under the relaxed smoothness assumption is new.
    \item Based on our theoretical analysis, we proposed a Variant of S2P (VS2P), which exploits our new convergence properties and incorporates our theoretical findings. Particularly, a "clipping" mechanism naturally emerges to the approximated gradient, which distinguishes it from existing zeroth-order methods.
    \item We conduct extensive experiments on large deep models, including image and language models, that show VS2P outperforms or achieves competitive performance compared to competing methods on gradient-free adaptation. 
\end{itemize}

\section{Related work}
Extensive existing literature studied the zeroth-order
optimization under convex and non-convex settings~\cite{shamir2017optimal, jamieson2012query, agarwal2009information, raginsky2011information, duchi2015optimal}. 
Bounds to reach first-order stationary points under general smoothness assumption have been derived, which generally depend on model parameter dimension $d$~\cite{nesterov2017random, bergou2020stochastic}.

A line of work investigates the effectiveness of noise perturbation to various tasks, e.g., generalizing Gaussian Smoothing to Bernoulli($\pm 1$) distribution~\cite{gao2022generalizing}, orthonormalization of noise perturbation over Gram–Schmidt process~\cite{choromanski2018structured, maheswaranathan2019guided}.
Moreover, practical and theoretical results showed the advantages of the zeroth-order method meeting low-rank structures of the underlying problem~\cite{cai2022zeroth, malladi2023fine, wang2018stochastic, sener2020learning}.
Some approaches also guarantee second-order convergence~\cite{lucchi2021second, zhang2022faster, ren2023escaping}. 
However, the problem has rarely been studied under the popular relaxed smoothness assumption~\cite{zhang2019gradient}.

Based on the theories, many work proposed practical methods to adapt to various deep model scenarios such as hyper-parameter optimization~\cite{bergstra2012random, yang2020hyperparameter}, black-box adversarial attack on deep models~\cite{ilyas2018black, guo2019simple, liu2018signsgd}.
Moreover, several methods have been developed for and adapted to deep models gradient-free adaptation~\cite{malladi2023fine, prasad2022grips, deng2022rlprompt}.

\section{Stochastic Two-Point Search}
In this section, we first introduce a prototype of Stochastic Two-Point search (S2P) and analyze its convergence using the general smoothness assumption. 
We then improve our analysis of S2P using the relaxed smoothness assumption, which leads to the Variant of Stochastic Two-Point Search (VS2P). 

Throughout this paper, we use bold lowercase letters $\rvx, \rvy$ to denote vectors. For vectors, we use $||\cdot||$ to denote the $\ell_{2}$-norm. For a function $f:\mathbb{R}^{d}\rightarrow \mathbb{R}$, we use $\nabla f$ to denote the gradient and $f^{\star}$ to denote the global minimum of function $f$.
We use $\mathcal{O(\cdot)}, \Omega(\cdot)$ to hide absolute constants that do not depend on any problem parameter.
We need the following standard definitions and assumptions~\cite{nesterov2017random, bergou2020stochastic, zhang2019gradient}.

\begin{definition}
    For a differentiable function $f$, $\rvx$ is a \textbf{$\epsilon$-first-order stationary point} if $||\nabla f(\rvx)|| \leq \epsilon$.
\end{definition}
\begin{definition}
    A differentiable function $f$ is \textbf{$L$-gradient Lipschitz} if
    $
    ||\nabla f(\rvx_{1}) - \nabla f(\rvx_{2})|| \leq 
    L||\rvx_{1} - \rvx_{2}||\;\;\;\; \forall \rvx_{1},\rvx_{2}.
    $
\end{definition}
\begin{definition}
    A differentiable function $f$ is \textbf{$(L_{0}, L_{1})$-smoothness} if
    $
    ||\nabla^{2}f(\rvx)||\leq L_{0} + L_{1}||\nabla f(\rvx)||.
    $
\end{definition}
\begin{assumption}
\label{ass_lsmooth_gd}
    The function $f$ is $L$-gradient Lipschitz.
\end{assumption}
\begin{assumption}
\label{ass_l0l1}
    The function $f$ satisfies $(L_{0}, L_{1})$-smoothness
\end{assumption}
Unless otherwise specified, we assume function $f$ is bounded below by $f^{\star}$.

\subsection{Stochastic Two-Point Search (S2P)}
We first propose a prototype algorithm, Stochastic Two-Point search (S2P), which improves STP by removing the non-updating component, $f(\rvx_{k})$. 
The seemingly minor change is non-trivial because the computation of $f(\rvx_{k})$ in STP cannot be reused from the previous iteration under the batch data condition, and is critical to the convergence of STP.
If a similar convergence is maintained in S2P, such elimination of an additional forward pass can greatly reduce the computation needed to optimize deep models. The S2P algorithm is summarized in Algorithm~\ref{alg_algorithm_step_size}.

Specifically, the choice of the distribution of random perturbations within three commonly used probability distributions, normal, uniform, and Rademacher distribution (Bernoulli $\pm1$ distribution), does not alter our analysis results within our proof framework. 
However, we use the random perturbations from the Rademacher distribution for our analysis since STP originally utilizes the normal distribution and uniform distribution.
%
%

\begin{algorithm}[tb]
    \caption{Stochastic Two-Point search (S2P). }
    \label{alg_algorithm_step_size}
    \textbf{Inputs}: Epochs $K$, objective function $f$ parameterized with $\rvx\in\mathbb{R}^{d}$, stopping criterion $\epsilon$.
    \begin{algorithmic}[1] 
        \FOR{ $k =0,...,K$}
        \STATE $\rvs_{k}\sim \mathcal{R}$ \COMMENT{Rademacher distribution. The normal and uniform distribution also apply.}
        \STATE Choosing one from Option 1-4: Update $\alpha$
        \STATE \;\;\;\;Option 1. $\alpha_{k} = \alpha_{0} / \sqrt{Kd}$ \COMMENT{Theorem~\ref{theorem_stationary}}
        \STATE \;\;\;\;Option 2. $\alpha_{k} = \frac{|\gamma_{k}|}{Ld}$ where $|\gamma_{k}| = \frac{|f(\rvx + \rho\rvs_{k}) - f(\rvx - \rho\rvs_{k})|}{2\rho}$ \COMMENT{Theorem~\ref{theorem_stationary}}
        \STATE \;\;\;\;Option 3. $\alpha_{k} = \sqrt{2}/BL_{1}\sqrt{dK}$ \COMMENT{Theorem~\ref{theorem_l0l1_query_complexsity}}
        \STATE \;\;\;\;Option 4. $\alpha_{k} = \frac{|\gamma_{k}|}{(AL_{0} + \sqrt{2}BL_{1}|\gamma_{k}|)d}$ where $|\gamma_{k}| = \frac{|f(\rvx + \rho\rvs_{k}) - f(\rvx - \rho\rvs_{k})|}{2\rho}$\COMMENT{Theorem~\ref{theorem_l0l1_query_complexsity}}
        \STATE $\rvx_{k+1} = \argmin \{f(\rvx_{k} + \alpha_{k}\rvs_{k}), f(\rvx_{k} - \alpha_{k}\rvs_{k})\}$ 
        \ENDFOR
        \STATE \textbf{return} $\rvx$
    \end{algorithmic}
\end{algorithm}

\subsection{S2P under General Smoothness Assumption}
\label{subsection_general_smooth}

We first analyze the convergence properties of $f$ running the proposed S2P algorithm under the general smoothness assumption.
Similar to STP, we initiate our analysis from Lemma~\ref{lemma_sign-g_sign-s}, which shows the absolute value of the inner product between gradient $\rvg$ and random perturbation $\rvs$ is larger than a positive value in expectation, which forms the foundation of descent.
Building upon this foundation, Lemma~\ref{lemma_progress} introduces a progressive bound and identifies the optimal step size at each iteration. 
%
The central result in this subsection is Theorem~\ref{theorem_stationary}, which establishes that Algorithm~\ref{alg_algorithm_step_size} can employ both stationary and dynamic step sizes (Option 1 and Option 2, respectively) to reach an $\epsilon$-first-order stationary point with a query complexity of $\mathcal{O}(\frac{d}{\epsilon^{2}})$.

%


\begin{lemma}
\label{lemma_sign-g_sign-s}
For all $\rvg\in\mathbb{R}^{d}$, and random vector $\rvs \sim \mathcal{R}$ where $\mathcal{R}$ is the Rademacher distribution, then
$
\mathbb{E}_{\rvs\sim \mathcal{R}}|\langle\rvg, \rvs\rangle| \geq \frac{1}{\sqrt{2}}||\rvg||_{2}.
$
\end{lemma}
The result can be directly derived by applying Khintchine inequality~\cite{khintchine1923dyadische}, and the proof is presented in the appendix~\ref{append_section_tech}. 
Please refer to Lemma 3.4 in~\cite{bergou2020stochastic} for similar results with normal\&unifrom distributions. 
Note that the random perturbation can be normalized as done in STP, so we have $\mathbb{E}_{\rvs\sim \mathcal{R}}|\langle\rvg, \frac{\rvs}{||\rvs||}\rangle| = \frac{1}{\sqrt{d}}\mathbb{E}_{\rvs\sim \mathcal{R}}|\langle\rvg, \rvs\rangle| \geq \frac{1}{\sqrt{2d}}||\rvg||_{2}$. 
The formula trick can be easily applied to the following analysis, and the conclusion remains the same.

\begin{lemma}[Progressive bound]
\label{lemma_progress}
Suppose objective function $f(\cdot)$ satisfies Assumption~\ref{ass_lsmooth_gd} and $||\nabla f(\rvx_{k})||_{2}\geq \epsilon_{g}$. If we run algorithm~\ref{alg_algorithm_step_size} with step size $\alpha = \frac{\sqrt{2}\epsilon_{g}}{2Ld}$, we have following progressive bound
$
\mathbb{E}[f(\rvx_{k+1}) - f(\rvx_{k})|\rvx_{k}] \leq -\Omega(\frac{\epsilon_{g}^{2}}{Ld}),
$ where $\mathbb{E}[\cdot|\rvx_{k}]$ denotes the conditional expectation w.r.t. $\rvx_{k}$.
\end{lemma}
The proof is presented in the appendix~\ref{appen_cor_stsp}.

\begin{theorem}[Query complexity]
\label{theorem_stationary}
Suppose objective function $f(\cdot)$ satisfies Assumption~\ref{ass_lsmooth_gd}. If we run algorithm~\ref{alg_algorithm_step_size} with step size strategy options 1 or 2, the algorithm returns in expectation an $\epsilon$-first-order stationary point in $\mathcal{O}(\frac{d}{\epsilon^{2}})$ function evaluations. Specifically,
\begin{align*}
    \text{Option 1}\;\;\;\;K &\geq \frac{2d}{\epsilon^{2}}(\frac{(f(\rvx_{0}) - f^{\star})}{\alpha_{0}} + \frac{L\alpha_{0}}{2})^{2},\;\;\;\; \\
    \text{Option 2}\;\;\;\;K &\geq \frac{4Ld (f(\rvx_{0}) - f^{\star})}{\epsilon^{2} - \frac{\rho^{2}}{2}},
\end{align*}
where $\alpha_{0}> 0$ for Option 1 stationary step size; For Option 2 dynamic step size, scalar $\rho_{k} \in (0, \frac{\sqrt{2}||\nabla f(\rvx_{k})||}{Ld}]$ for $\rho_{k}$ in each iteration. 
Generally, it can be set to a small value, e.g., $\rho = \frac{\sqrt{2}\epsilon}{Ld}$.
\end{theorem}
The proof is presented in the appendix~\ref{appen_cor_stsp}.

Especially, the strategy of dynamic step size (Option 2 in Algorithm~\ref{alg_algorithm_step_size}) aims to approximate the (worst-case) optimal step size at each iteration, i.e., approximating $\alpha_{k}^{\text{opt}} = \frac{|\nabla f(\rvx_{k})^{T}\rvs_{k}|}{Ld}$ with $\alpha_{k} = \frac{|\gamma_{k}|}{Ld}$ where $|\gamma_{k}| = \frac{|f(\rvx + \rho\rvs_{k}) - f(\rvx - \rho\rvs_{k})|}{2\rho}$. 
Simultaneously, the error $|\delta_{k}| := |\alpha_{k} - \alpha_{k}^{\text{opt}}| \leq \frac{\rho}{2}$ is bounded (Please refer to~\cref{append_eq_lr_error} in Appendix~\ref{append_theorem_stationary} to see how $\rho$ control the error $|\delta_{k}|$.)
The above findings underline the fundamental correlation between $\alpha_{k}^{\text{opt}}$ and $|\gamma_{k}|$ under general smoothness assumption, specifically, $\alpha_{k}^{\text{opt}} \propto |\gamma_{k}|$ with a sufficient small $\rho$.

We also note that S2P with dynamic step size strategy involves two different symmetric perturbations in each iteration, i.e., option 2 and line-8 in Algorithm~\ref{alg_algorithm_step_size}. This approach necessitates twice the computational cost in each update compared to MeZO in practical deployment. Ultimately, our goal is to achieve one symmetric perturbation per iteration in our proposed variant of S2P, i.e., VS2P in Algorithm~\ref{alg_algorithm_step_size_v}.

\begin{remark}[Inherent connection between S2P and MeZO]
\label{remark_general}
Under the general smoothness assumption, S2P employing Option 2 in Algorithm~\ref{alg_algorithm_step_size} has the same updating formula with MeZO when the sign trick (assigning $\text{sign}(f(\rvx + \rho\rvs_{k}) - f(\rvx - \rho\rvs_{k}))$ as the sign of $f(\rvx + \alpha_{k}\rvs_{k}) - f(\rvx - \alpha_{k}\rvs_{k})$, described in Section~\ref{subsection_vs2p}) is applied. I.e., updating formula of S2P employing option 2 and the sign trick: $\rvx_{k+1} =\rvx_{k} - \alpha_{k} \text{sign}\big(f(\rvx + \rho\rvs_{k}) - f(\rvx - \rho\rvs_{k})\big)\frac{|f(\rvx + \rho\rvs_{k}) - f(\rvx - \rho\rvs_{k})|}{2\rho}\rvs_{k} = \rvx_{k} - \alpha_{k}\frac{f(\rvx + \rho\rvs_{k}) - f(\rvx - \rho\rvs_{k})}{2\rho}\rvs_{k}$, where the $\alpha_{k}$ absorbs the unknow problem properties. Thus, the derived formula is the same as the updating formula of MeZO (principally equivalent to GA as mentioned). However, the sign trick can only be applied conditionally, which may imply the flaw of MeZO and the necessity of connecting $\rho$ and $\alpha_{k}$.
\end{remark}

\subsection{S2P under Relaxed Smoothness Assumption}
We now analyze the convergence properties of the proposed S2P algorithm under the relaxed smoothness assumption. 
The assumption posits that $f$ may behave like a smooth function in certain local regions of the loss landscape, but there can also exist some highly non-smooth regions where the top eigenvalue of Hessian may be large, necessitating special considerations~\cite{zhang2019gradient, kunstner2023noise}.

Lemma~\ref{lemma_l0l1_progress} provides the progressive bound and the optimal step size at each iteration. 
We highlight that $\epsilon_{g}$ is no longer linearly dependent on the step size $\alpha$, which distinguishes this result from the one in Lemma~\ref{lemma_progress}.
Intuitively, a large $\epsilon_{g}$ indicates a large gradient, which may, in turn, imply a large top eigenvalue of Hessian under Assumption~\ref{ass_l0l1}. 
Consequently, a large step size is no longer the best choice. 
This concept is pivotal in further improvements upon S2P.

The main result in this subsection is Theorem~\ref{theorem_l0l1_query_complexsity}, which shows that Algorithm~\ref{alg_algorithm_step_size} can employ both stationary and dynamic step sizes (Option 3 and Option 4, respectively) to reach an $\epsilon$-first-order stationary point with a query complexity of $\mathcal{O}(\frac{d}{\epsilon^{2}})$.
Importantly, Theorem~\ref{theorem_l0l1_query_complexsity} shows the structured nature within the learning process when taking dynamic step size. For instance, in regions where the function is smooth and the gradient norm is large, we can anticipate a reduced query complexity. 
Conversely, under the fourth condition outlined in Table~\ref{table1_l0l1_dynamic_summary}, we encounter situations where it is impossible to decrease $||\nabla f(\rvx)||$ due to high levels of non-smoothness. 
Fortunately, our proposed step size strategy allows us to safely traverse these highly non-smooth regions.

%
%

\begin{lemma}[Progressive bound]
\label{lemma_l0l1_progress}
Suppose objective function $f(\cdot)$ satisfies Assumption~\ref{ass_l0l1} and $||\nabla f(\rvx_{k})||\geq \epsilon_{g}$. Algorithm~\ref{alg_algorithm_step_size} with step size $\alpha = \frac{\sqrt{2}\epsilon_{g}}{2(AL_{0} + BL_{1}\epsilon_{g})d}$ gives the following following progressive bound
$
\mathbb{E}[f(\rvx_{k+1}) - f(\rvx_{k})|\rvx_{k}] \leq -\Omega(\frac{\epsilon_{g}^{2}}{(AL_{0} + BL_{1}\epsilon_{g})d}),
$ where $\mathbb{E}[\cdot|\rvx_{k}]$ denotes the conditional expectation w.r.t. $\rvx_{k}$, and constants $A = 1.01, B = 1.01$.
\end{lemma}
The proof is presented in the appendix~\ref{append_subsection_relaxed_progressive}.

\begin{theorem}[Query complexity]
\label{theorem_l0l1_query_complexsity}
Suppose objective function $f(\cdot)$ satisfies Assumption~\ref{ass_l0l1}. With step size strategy options 3 or 4, Algorithm~\ref{alg_algorithm_step_size} returns in expectation an $\epsilon$-first-order stationary point in $\mathcal{O}(\frac{d}{\epsilon^{2}})$ function evaluations. 
Specifically,
\begin{align*}
    \text{Option 3}\;\;&K\geq (\sqrt{d} + \frac{AL_{0}\sqrt{d} + BL_{1}(f(\rvx_{0}) - f^{\star})\sqrt{d}}{\epsilon})^{2}\\
    \text{Option 4}\;\;&\text{The result is summarized in Table~\ref{table1_l0l1_dynamic_summary}}.
\end{align*}
where constants $A = 1.01, B = 1.01$.
\end{theorem}
The proof is presented in the appendix~\ref{append_subsection_relaxed_query_complexsity}.

\bgroup
\setlength\tabcolsep{0.2em}
\begin{table*}[!ht]
\begin{center}
\small
\begin{threeparttable}
    \begin{tabular}{ccc}
    \toprule
         Conditions$^{[b]}$ & requirement over $\rho^{[a]}$ & Query complexity \\
        \cmidrule(r){1-3}
          $L_{1}\leq \frac{1}{\sqrt{2}B}$, $||\nabla f(\rvx)|| \geq \frac{AL_{0}}{1-\sqrt{2}BL_{1}}$ 
          &$\rho\leq \frac{1}{d\sqrt{2\xi\sqrt{d}}}$ &$\frac{8d(f(\rvx_{0}) - f^{\star})}{\epsilon}$\\
          $L_{1}\leq \frac{1}{\sqrt{2}B}$, $||\nabla f(\rvx)|| \leq \frac{AL_{0}}{1-\sqrt{2}BL_{1}}$   
          &$\rho \leq \frac{1}{d}\sqrt{\frac{\epsilon}{2\xi(A L_{0} + \sqrt{2}BL_{1}\epsilon)\sqrt{d}}}$ & $\frac{8AL_{0}d(f(\rvx_{0}) - f^{\star})}{(1 -\sqrt{2}B L_{1})\epsilon^{2}}$\\
           $L_{1}\geq \frac{1}{\sqrt{2}B}$, $||\nabla f(\rvx)|| \leq \frac{AL_{0}}{\sqrt{2}BL_{1} - 1}$
           &$\rho \leq \frac{1}{d}\sqrt{\frac{\epsilon}{2\xi(A L_{0} + \sqrt{2}BL_{1}\epsilon)\sqrt{d}}}$ & $\frac{8AL_{0}d(f(\rvx_{0}) - f^{\star})(2\sqrt{2}BL_{1} - 1)}{(\sqrt{2}BL_{1} - 1)\epsilon^{2}}$\\
           \cmidrule(r){1-3}
           $L_{1}\geq \frac{1}{\sqrt{2}B}$, $||\nabla f(\rvx)|| \geq \frac{AL_{0}}{\sqrt{2}BL_{1} - 1}$
           &$\rho \leq \frac{1}{d}\sqrt{\frac{\epsilon}{2\xi(A L_{0} + \sqrt{2}BL_{1}\epsilon)\sqrt{d}}}$ & $\frac{8(2\sqrt{2}BL_{1} - 1)(\sqrt{2}BL_{1} - 1)(f(\rvx_{0}) - f^{\star} - \epsilon)d}{AL_{0}}$\\
    \bottomrule
    \end{tabular}
    \begin{tablenotes}
        \item $^{[a]}$ $\xi$ is a constant associated with third-order property of $f$, detailed in appendix~\cref{append_l0l1_dy_1}.
        \item $^{[b]}$ For the fourth condition, decreasing loss value instead of gradient norm, detailed in appendix~\cref{append_l0l1_condition3_1}.
    \end{tablenotes}
    \caption{With dynamic step size strategy, the convergence property under relaxed smoothness.}
    \label{table1_l0l1_dynamic_summary}
\end{threeparttable}
\end{center}
\end{table*}
\egroup

Similarly, the strategy of dynamic step size (Option 4 in Algorithm~\ref{alg_algorithm_step_size}) aims to approximate the optimal step size at each iteration, i.e., approximating $\alpha_{k}^{\text{opt}} = \frac{|\rvs^{T}\nabla f(\rvx_{k})|}{(A L_{0} + \sqrt{2}BL_{1}|\rvs^{T}\nabla f(\rvx_{k})|)d}$ with $\alpha_{k} = \frac{|\gamma_{k}|}{(A L_{0} + \sqrt{2}BL_{1}|\gamma_{k}|)d}$ where $|\gamma_{k}| = \frac{|f(\rvx + \rho\rvs_{k}) - f(\rvx - \rho\rvs_{k})|}{2\rho}$. 
Simultaneously, the error $|\delta_{k}| := |\alpha_{k} - \alpha_{k}^{\text{opt}}| \leq \xi\rho^{2}d^{3/2}$ is bounded, where $\xi$ is a constant associated with third-order property of $f$. Please refer to proofs after~\cref{append_l0l1_dy_1} in Appendix~\ref{append_subsection_relaxed_query_complexsity}.
The above findings underline the fundamental correlation between $\alpha_{k}^{\text{opt}}$ and $|\gamma_{k}|$ under relaxed smoothness assumption, specifically, $\alpha_{k}^{\text{opt}} \propto \frac{1}{1/|\gamma_{k}| + 1/C}$ with some constant $C$ and a sufficient small $\rho$.

\begin{remark}[Step size constrain]
\label{remark_relax}
According to studies of clipped gradient descent under relax smoothness assumption~\cite{zhang2019gradient, kunstner2023noise}, when traversing the non-smooth regions of loss landscape where the top eigenvalue of Hessian might be large, it is necessary to constrain the step size according to the upper bound of the top eigenvalue of Hessian. In these scenarios, naive gradient descent may be too aggressive. A similar situation arises in zeroth-order optimization. Given the "approximated gradient" $\gamma_{k} = \frac{f(\rvx + \rho\rvs_{k}) - f(\rvx - \rho\rvs_{k})}{2\rho}$, MeZO suggests $\alpha_{k} \propto |\gamma_{k}|$, while S2P suggests $\alpha_{k} \propto \frac{1}{1/|\gamma_{k}| + 1/C}$. The term $\frac{1}{1/|\gamma_{k}| + 1/C}$ can be considered as a differential version of clipping, where $C$ is a problem-dependent property that practically functions as an inhibition strength to $|\gamma_{k}|$.
\end{remark}

\begin{remark}[Comparison between stationary step size and dynamic step size]
\label{remark_comp_two_step_size}
We derive the upper bound of query complexities when employing either stationary step size or dynamic step size, under both general and relaxed smoothness assumptions. Our observations indicate that the query complexity for algorithms using a dynamic step size is smaller than that for algorithms using a stationary step size. This suggests that algorithms exploiting a dynamic step size strategy are expected to outperform those with a stationary step size strategy. Intuitively, the dynamic step size will approximate the optimal step size at each parameter update step, thus leading to a larger progress at each step.
\end{remark}


\begin{algorithm}[H]
    \caption{Variant of Stochastic Two-Point search (VS2P).}
    \label{alg_algorithm_step_size_v}
    \textbf{Inputs}: Epochs $K$, objective function $f(\cdot)$ parameterized with $\rvx\sim \mathbb{R}^{d}$,   smoothing parameter $\rho$, learning rate $\eta$, Decay strategy, e.g., cosine decay. Scalars $\tau_{a} = \tau_{b}=3$.\\
    \textbf{Parameter}: $\rvx$
    \begin{algorithmic}[1] 
        \FOR{ $k =0,...,K$}
        \STATE $\rvs_{k}\sim \mathcal{R}$ \COMMENT{Rademacher distribution. Normal and uniform distribution also apply.}
        \STATE $|\gamma_{k}| = |\frac{f(\rvx_{k} + \rho\rvs_{k}) - f(\rvx_{k} - \rho\rvs_{k})}{2\rho}|, \beta_{k} = \text{sign}\big(f(\rvx_{k} + \rho\rvs_{k}) - f(\rvx_{k} - \rho\rvs_{k})\big)$
        \STATE $\sigma = \text{Std Dev}(\gamma_{\text{recent}})$
        \STATE $\alpha_{k} = \text{Decay strategy}(\eta, k)\times\frac{\beta_{k}\rho}{\tau_{b}\sigma/|\gamma_{k}| + \tau_{b}/\tau_{a}}$
        \STATE $\rvx_{k+1} = \rvx_{k} + \alpha_{k}\rvs_{k}$ 
        \ENDFOR
        \STATE \textbf{return} $\rvx$
    \end{algorithmic}
\end{algorithm}

\subsection{Variant of Stochastic Two-Point Search (VS2P)}
\label{subsection_vs2p}
Our convergence analysis of $f$ running S2P under the general and relaxed smoothness assumptions yields insights into a Variant of S2P (VS2P), which is a practical practice of 
Algorithm~\ref{alg_algorithm_step_size}.
VS2P adopts the dynamic step size strategy of stochastic two-point search and has two highlighted improvements.


\textbf{$\gamma$-clipping.} 
The immediate observation stemming from the convergence properties of $f$ running S2P with dynamic step size under relaxed smoothness assumption is the non-linear dependence between the approximated optimal step size $\alpha_{k}$ and $|\gamma_{k}|$, i.e., $\alpha_{k} = \frac{|\gamma_{k}|}{(AL_{0} + \sqrt{2}BL_{1}|\gamma_{k}|)d} = \frac{1}{(AL_{0}/|\gamma_{k}| + \sqrt{2}BL_{1})d}$. 
Specifically, the step size is almost linearly incremental when $|\gamma_{k}|$ is small, but the increment decreases fast when $|\gamma_{k}|$ is relatively large. 
We thus propose a strategy to mimic similar behavior, i.e., $\alpha_{k} \propto \gamma_{k}^{\prime}:=\frac{1}{1/|\gamma_{k}| + 1/\tau_{a}\sigma}$. 
The term $\tau_{a}\sigma$, where $\sigma:=\text{Std Dev}(\gamma_{\text{recent}})$, practically act as the threshold to estimate the inhibition strength to $|\gamma_{k}|$. 

\textbf{Sign trick} 
Having $\alpha_{k}\propto \gamma_{k}^{\prime}$, then we analyze the sign of step size.
%
According to our analysis, the algorithm applying the dynamic step size strategy has the potential to outperform the algorithm with stationary step size. 
However, the dynamic step size strategy requires twice symmetric perturbations forward passes at each iteration $k$. 
In order to reduce the number of forward passes, we propose to assign $\beta_{k}: = \text{sign}(f(\rvx + \rho\rvs_{k}) - f(\rvx - \rho\rvs_{k}))$ as the sign of $f(\rvx + \alpha_{k}\rvs_{k}) - f(\rvx - \alpha_{k}\rvs_{k})$, which we call \emph{sign trick}.
Then the calculation of $\argmin\{f(\rvx + \alpha_{k}\rvs_{k}), f(\rvx - \alpha_{k}\rvs_{k})\}$ in Algorithm~\ref{alg_algorithm_step_size} is unnecessary since $\rvx + \beta_{k}\text{abs}(\alpha_{k})\rvs_{k} = \argmin\{f(\rvx + \alpha_{k}\rvs_{k}), f(\rvx - \alpha_{k}\rvs_{k})\}$ where $\text{abs}(\alpha_{k})$ is known in practical. 
However, for a safe sign assignment, principally, it is required $\alpha_{k} \leq \rho$ at least supposing $\rho$ is small enough to guarantee the consistency of sign of directional gradient in local regions.
Based on the above intuitions, we propose a strategy 
$\alpha_{k} \propto \beta_{k}\rho\gamma_{k}^{\prime}/\tau_{b}\sigma$ since $\gamma_{k}^{\prime}/\tau_{b}\sigma \leq 1$ almost for sure supposing $\tau_{b}$ is large enough.

Finally, we have $\alpha_{k} \propto \frac{\beta_{k}\rho}{\tau_{b}\sigma/|\gamma_{k}| + \tau_{b}/\tau_{a}}$ in short, which emphasizes (a) The non-linear dependence between $\alpha_{k}$ and $|\gamma_{k}|$; (b) The interaction between the absolute value and standard deviation of $\gamma_{k}$.
The complete strategy is elucidated Algorithm~\ref{alg_algorithm_step_size_v}. And, $\tau_{a}=\tau_{b}:=3$ according to three-sigma rule. 

\section{Experiments}
In this section, we compare the proposed VS2P 
with several standard methods, including MeZO~\cite{nesterov2017random, malladi2023fine} and STP~\cite{bergou2020stochastic} on image tasks and language tasks.
\textit{Image tasks} fully fine-tune pre-trained DenseNet121~\cite{huang2017densely}, ResNet18~\cite{he2016deep}, ViT-B/16~\cite{dosovitskiy2020vit}, and VGG11~\cite{simonyan2014very} under datasets MNIST~\cite{deng2012mnist} and CIFAR10~\cite{krizhevsky2009learning}.
\textit{Regrading language tasks}, we mainly follow the experimental settings described in~\cite{malladi2023fine} to fully fine-tune LLM model GPT2~\cite{radford2019language} on selected tasks of GLUE benchmark~\cite{wang2018glue}.
It is worth noting that this work only considers the zeroth-order methods, which enables fine-tuning deep models with only forward-pass computing environments.


%
%


\textbf{Setup} Unless otherwise specified, all methods are paired with a cosine learning rate scheduler, utilize normal distribution as the default source of random perturbations, and employ fixed smoothing parameter of $\rho=\text{1e-3}$. Thus, there is only one tunable hyper-parameter for the proposed method and baselines, i.e., the learning rate. We average the results across three random seeds for all the settings.

\textit{Learning rate selection}
For image tasks, we search for the best learning rate over some pre-defined lists for different methods.
Specifically, we search for the best learning rate from lists \textit{MeZO} - $\{$2.5e-4, 1e-4, 7.5e-5, 5e-5, 2.5e-5, 1e-5$\}$; \textit{STP} - $\{$2.5e-3, 1e-3, 7.5e-4, 5e-4, 2.5e-4, 1e-4$\}$; \textit{VS2P} - $\{$1, 1.5, 3, 5$\}$.
Regarding the language tasks, we search for the best learning rate of one specific task in the GLUE benchmark for different methods and directly apply the selected learning rate to the remaining tasks, considering the same backbone of language tasks and tuning efficiency.
Specifically, the lists of learning rates for language tasks are $0.1\times$ the lists of image tasks.
We report the performance of best learning rate of each method in Section~\ref{exp_comp_baselinse}.
The details of the setup are summarized in Appendix~\ref{append_exp_setup}.

It is worth noting that STP requires three function queries at each parameter update whereas other methods need two, i.e., STP requires more computational resources than other methods under the same number of training epochs. So we get the results of STP ending at 2/3 training period and re-scale to compare with other methods, denoting as STP$^{\prime}$ correspondingly.

\subsection{Performance comparison with standard methods}
\label{exp_comp_baselinse}

\textbf{Image tasks} 
Table~\ref{tab_img_test_acc} and Table~\ref{tab_img_training_acc} demonstrate the test accuracy and training acceleration ratio of the proposed VS2P over baselines under various network architectures and datasets MNIST \& CIFAR-10.
The results presented in Table~\ref{tab_img_test_acc} show that the proposed VS2P generally outperforms or achieves competitive performance with baselines. Besides, the proposed method usually speeds up the training process 1.5$\times$ compared with baselines, demonstrated as in Table~\ref{tab_img_training_acc}.

The dynamics of the training process including the training losses and evaluation accuracy along with varying epochs are summarized in Figure~\ref{fig_img_convergence_and_test_acc}. 
Accordingly, we summarize the performance of different methods under various learning rates in Appendix~\ref{append_perf_various_lr} to demonstrate the selection of the best learning rate.
Meanwhile, the results of training dynamics with 100 training epochs are summarized in Appendix~\ref{append_perf_img_task_100_epochs} and Appendix~\ref{append_perf_various_lr_100_epoch}, which show the same trends.

Note that we observe the failure case of MeZO in the VGG11\&MNIST setting. Searching within a broader range of learning rates, i.e., 1e-2$\sim$ 1e-6, does not mitigate the issue, referring to Appendix~\ref{append_failure_study}. We attribute this failure to the inherent flaw of the MeZO-like method under the relaxed smoothness assumption, although we lack theoretical evidence. Please refer to the further discussion in the \textit{Limitation} of Section~\ref{sec_conclusion}.

\textbf{Language tasks} The dynamics of the training process including the training loss, evaluation loss, and evaluation metric along with varying epochs are summarized in Figure~\ref{fig_lan_convergence}. It demonstrates that the proposed method outperforms other baselines with a relatively large margin on tasks $\{$QNLI, SST-2, STS-B$\}$ and achieves competitive performance on tasks $\{$MRPC, RTE$\}$.

\bgroup
\setlength\tabcolsep{0.2em}
\begin{table}[!ht]
    \centering
    \small
    \begin{threeparttable}
    \begin{tabular*}{\textwidth}{@{\extracolsep{\fill}} l *{9}{c} }
    \toprule
    \multirow{2}{*}{Algorithm}& \multicolumn{4}{c}{Test accuracy (\%) under MNIST} & \multicolumn{4}{c}{Test accuracy (\%) under CIFAR-10} \\
    \cmidrule(l{4pt}r{4pt}){2-5}
    \cmidrule(l{4pt}r{4pt}){6-9}
     & DenseNet121 & ResNet18 & VGG11 &  ViT-B/16 & DenseNet121 & ResNet18 & VGG11 &  ViT-B/16 \\
    \cmidrule{1-9}
     MeZO & 72.7±3.3  & 69.2±3.6 & 11.4±0.0 & \textbf{74.0±1.0} & 34.7±0.8 & 44.9±0.4 & \textbf{54.2±0.3} & 36.1±1.1\\
     STP$^{\prime}$ &  84.1±2.3 & 67.0±2.6 & 67.8±1.3 & 62.9±2.9 & 34.9±2.7 &  44.6±1.6 & 50.1±0.4 & 33.7±1.6\\
    VS2P & \textbf{86.6±2.4}  & \textbf{72.8±1.5} & \textbf{68.8±4.3} & 72.7±1.3 & \textbf{37.6±1.7} & \textbf{47.4±1.4} & 53.7±0.5 & \textbf{38.1±1.2}\\
    \bottomrule
    \end{tabular*}  
    \caption{Test accuracy under MNIST and CIFAR-10 \& 200 training epochs.}
    \label{tab_img_test_acc}
\end{threeparttable}
\end{table}
\egroup

\bgroup
\setlength\tabcolsep{0.2em}
\begin{table}[!ht]
    \centering
    \small
    \begin{threeparttable}
    \begin{tabular*}{\textwidth}{@{\extracolsep{\fill}} l *{9}{c} }
    \toprule
    Training & \multicolumn{4}{c}{Acceleration ratio under MNIST} & \multicolumn{4}{c}{Acceleration ratio under CIFAR-10} \\
    \cmidrule(l{4pt}r{4pt}){2-5}
    \cmidrule(l{4pt}r{4pt}){6-9}
    acc. over & DenseNet121 & ResNet18 & VGG11 &  ViT-B/16 & DenseNet121 & ResNet18 & VGG11 &  ViT-B/16 \\
    \cmidrule{1-9}
     MeZO &  4.2 & 1.6 & 10.5 & N/A & 2.0 & 1.7 & N/A & N/A\\
     STP$^{\prime}$ &  1.6 & 1.4 & N/A & 1.6 & 1.6 & 1.7 & 2.1 & 1.9\\
    \bottomrule
    \end{tabular*}  
    \caption{Training acceleration ratio of VS2P over baselines under MNIST and CIFAR-10 \& 200 training epochs. N/A denotes no training acceleration, please refer to Figure~\ref{fig_img_convergence_and_test_acc}.}
    \label{tab_img_training_acc}
\end{threeparttable}
\end{table}
\egroup

\begin{figure*}[!ht]
\label{fig_img_convergence_and_test_acc}
\begin{center}
\subfigure[Experiments under datasets MNIST. The selected learning rate is highlighted in line labels.]
{
\includegraphics[width=1.0\linewidth]{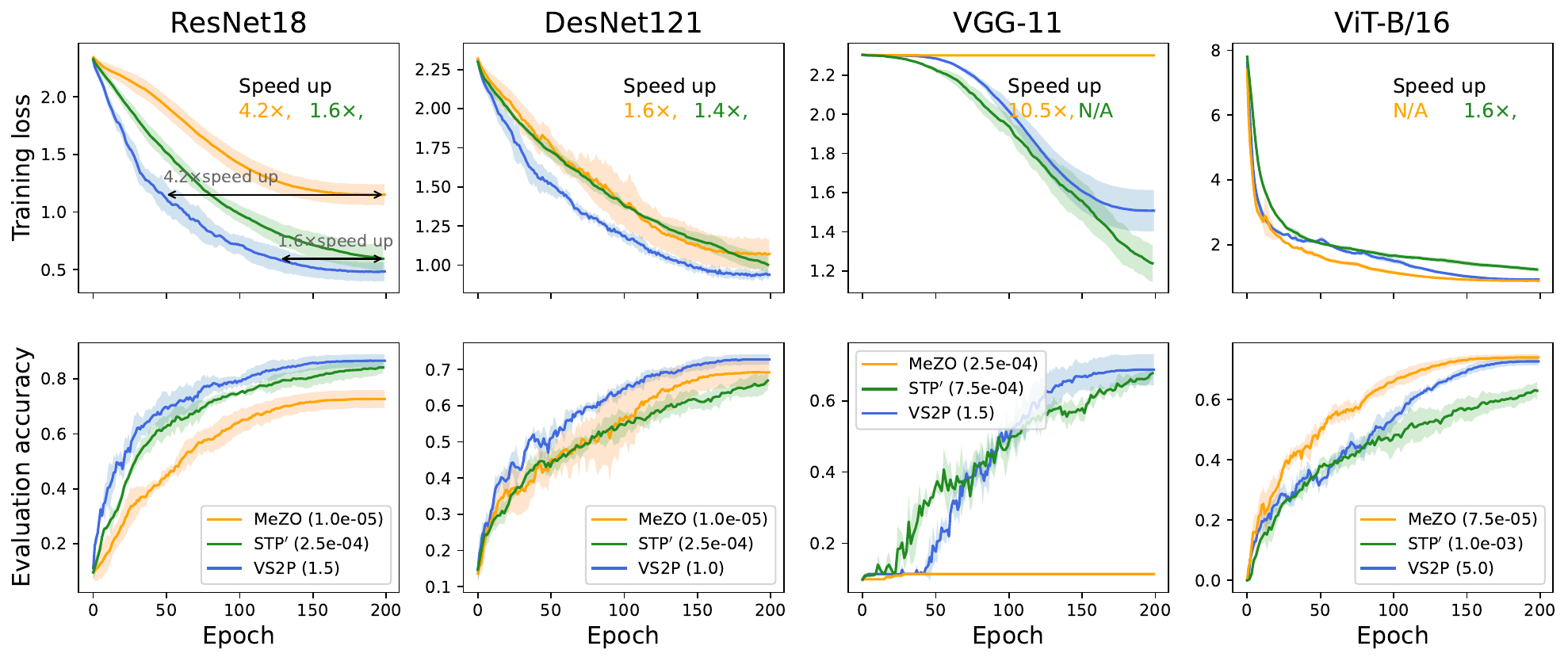}
}
\subfigure[Experiments under datasets CIFAR10. The selected learning rate is highlighted in line labels.]
{
\includegraphics[width=1.0\linewidth]{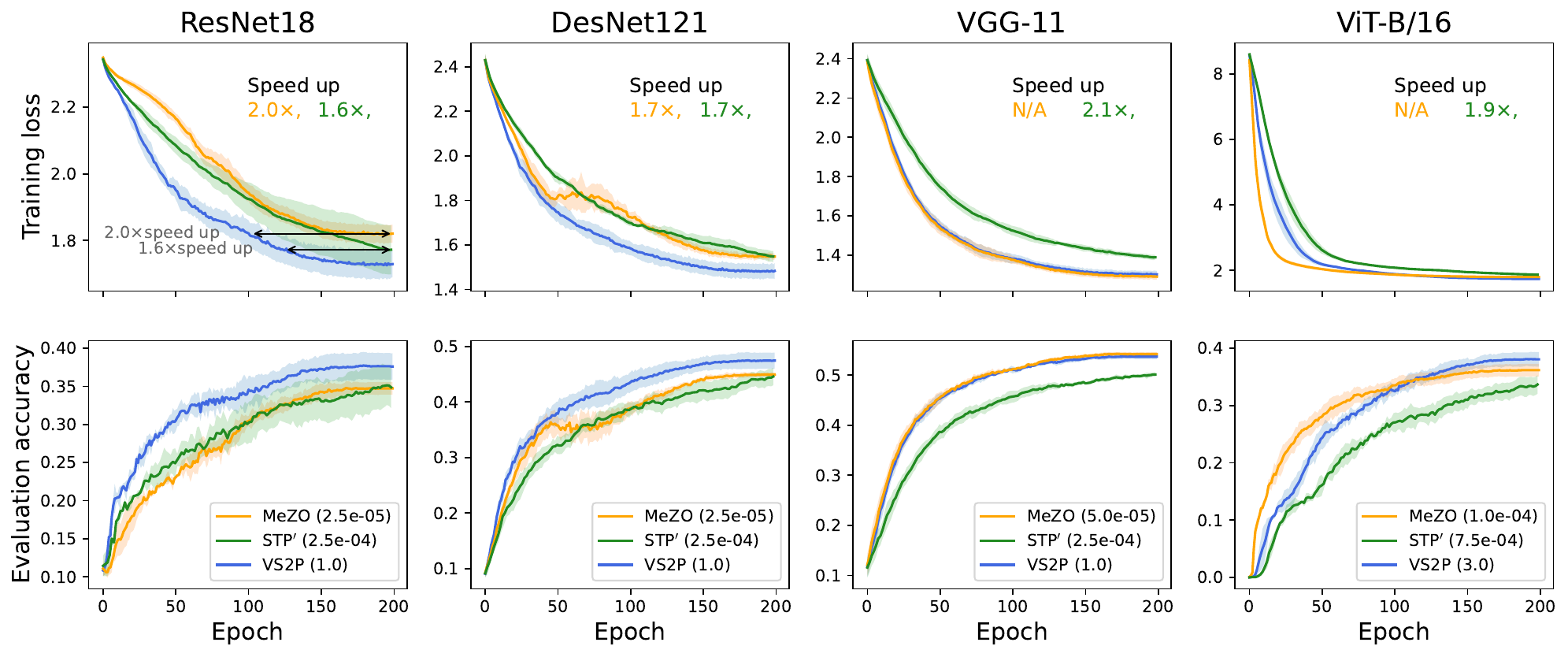}
}
\caption{Fully fine-tuning DenseNet121, ResNet18, ViT-B16, and VGG11 with proposed method and baselines under datasets MNIST and CIFAR10 \& 200 training epochs.}
\end{center}
\end{figure*}

\begin{figure*}[!ht]
\label{fig_lan_convergence}
\begin{center}
\subfigure
{
\includegraphics[width=1.0\linewidth]{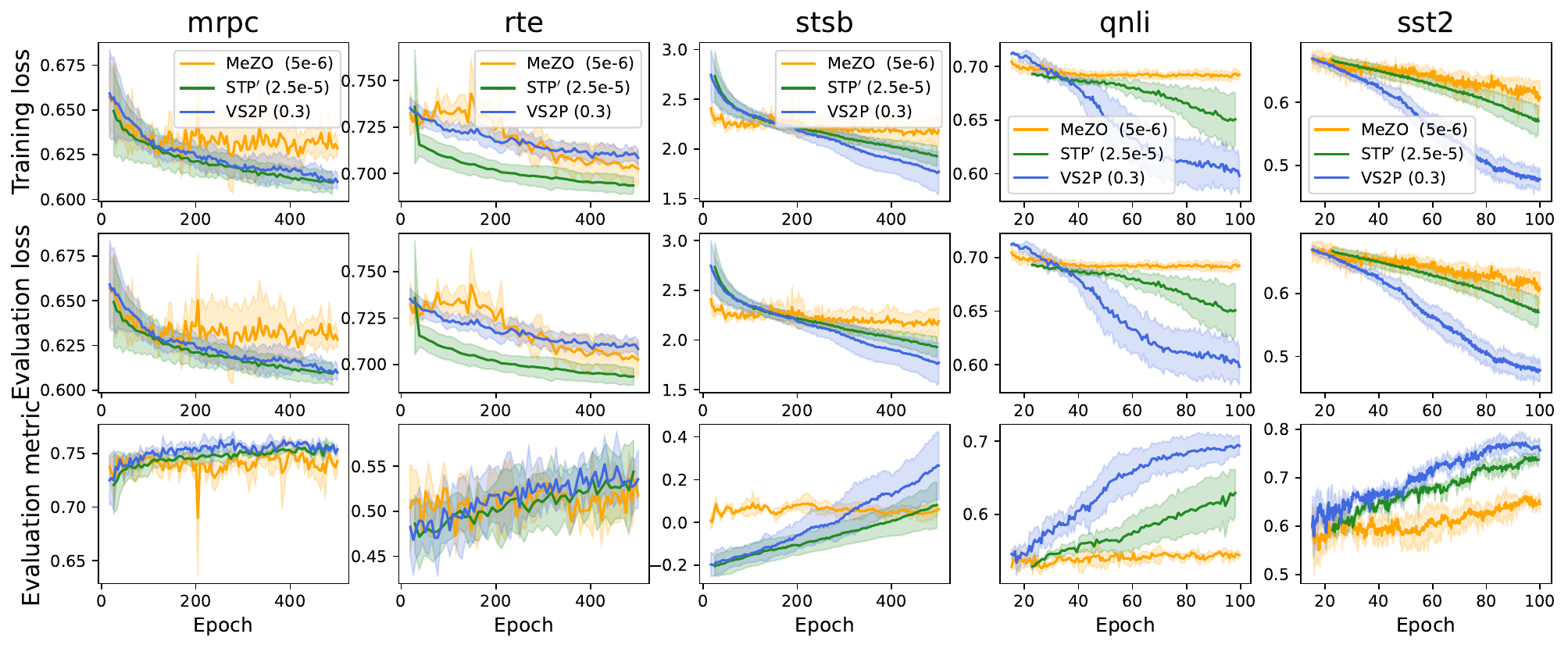}
\label{fig_convergence_llm}
}
\caption{Fine-tuning GPT2 under selected tasks of GLUE benchmark experiments. The selected learning rate is highlighted in line labels.}
\end{center}
\end{figure*}

\section{Conclusion}
\label{sec_conclusion}
In this work, we study the complexity of the proposed S2P method under both general and relaxed smoothness assumptions for zeroth-order optimization. 
Our theoretical analysis induces a variant of S2P, VS2P, which leverages our new theoretical findings and the practical practice, which we call sign trick.
This includes the non-linear dependence between the approximated optimal step size $\alpha_{k}$ and $|\gamma_{k}|$ when employing the dynamic step size strategy under relaxed smoothness assumption, as well as the practical implementation of the sign trick.
Empirical experiments demonstrate that the proposed VS2P outperforms or achieves competitive performance compared to standard methods across a spectrum of model types and scales.

%


\textbf{Future work} 
Further, this work emphasizes the integration between $\alpha_{k}$, $|\gamma_{k}|$ and $\text{Std Dev}(\gamma_{\text{recent}})$. 
It is non-trivial to consider capturing the interactions through the learning process to avoid tuning hyper-parameters, especially learning rate $\eta$, for various $f$. 
This investigation will lead us to our future work.

\textbf{Limitation} 1). Despite the theoretical guarantees of the proposed method under the relaxed smoothness assumption, we desire a lower bound for the MeZO-like method under the same relaxed smoothness assumption to possibly demonstrate the necessity of $\gamma$-clipping. Similar results, such as vanilla SGD being arbitrarily slower compared to clipped SGD under the relaxed smoothness assumption, have been observed in the previous work~\cite{zhang2019gradient}, paving the way for future research. 2). Our experiments show relatively consistent performance on image tasks and mixed results on language model training, warranting further investigations.

\section*{Acknowledgments}
This work was supported by the National Science Foundation under Grant IIS-1749940, IIS-2212174, Office of Naval Research N00014-20-1-2382, and National Institute on
Aging RF1AG072449.

\bibliographystyle{unsrtnat}
\bibliography{references}  

\begin{thebibliography}{54}
\providecommand{\natexlab}[1]{#1}
\providecommand{\url}[1]{\texttt{#1}}
\expandafter\ifx\csname urlstyle\endcsname\relax
  \providecommand{\doi}[1]{doi: #1}\else
  \providecommand{\doi}{doi: \begingroup \urlstyle{rm}\Url}\fi

\bibitem[Devlin et~al.(2018)Devlin, Chang, Lee, and Toutanova]{devlin2018bert}
Jacob Devlin, Ming-Wei Chang, Kenton Lee, and Kristina Toutanova.
\newblock Bert: Pre-training of deep bidirectional transformers for language understanding.
\newblock \emph{arXiv preprint arXiv:1810.04805}, 2018.

\bibitem[Bommasani et~al.(2021)Bommasani, Hudson, Adeli, Altman, Arora, von Arx, Bernstein, Bohg, Bosselut, Brunskill, et~al.]{bommasani2021opportunities}
Rishi Bommasani, Drew~A Hudson, Ehsan Adeli, Russ Altman, Simran Arora, Sydney von Arx, Michael~S Bernstein, Jeannette Bohg, Antoine Bosselut, Emma Brunskill, et~al.
\newblock On the opportunities and risks of foundation models.
\newblock \emph{arXiv preprint arXiv:2108.07258}, 2021.

\bibitem[Lester et~al.(2021)Lester, Al-Rfou, and Constant]{lester2021power}
Brian Lester, Rami Al-Rfou, and Noah Constant.
\newblock The power of scale for parameter-efficient prompt tuning.
\newblock \emph{arXiv preprint arXiv:2104.08691}, 2021.

\bibitem[Hu et~al.(2021)Hu, Shen, Wallis, Allen-Zhu, Li, Wang, Wang, and Chen]{hu2021lora}
Edward~J Hu, Yelong Shen, Phillip Wallis, Zeyuan Allen-Zhu, Yuanzhi Li, Shean Wang, Lu~Wang, and Weizhu Chen.
\newblock Lora: Low-rank adaptation of large language models.
\newblock \emph{arXiv preprint arXiv:2106.09685}, 2021.

\bibitem[Malladi et~al.(2023)Malladi, Gao, Nichani, Damian, Lee, Chen, and Arora]{malladi2023fine}
Sadhika Malladi, Tianyu Gao, Eshaan Nichani, Alex Damian, Jason~D Lee, Danqi Chen, and Sanjeev Arora.
\newblock Fine-tuning language models with just forward passes.
\newblock \emph{arXiv preprint arXiv:2305.17333}, 2023.

\bibitem[Prasad et~al.(2022)Prasad, Hase, Zhou, and Bansal]{prasad2022grips}
Archiki Prasad, Peter Hase, Xiang Zhou, and Mohit Bansal.
\newblock Grips: Gradient-free, edit-based instruction search for prompting large language models.
\newblock \emph{arXiv preprint arXiv:2203.07281}, 2022.

\bibitem[Nesterov and Spokoiny(2017)]{nesterov2017random}
Yurii Nesterov and Vladimir Spokoiny.
\newblock Random gradient-free minimization of convex functions.
\newblock \emph{Foundations of Computational Mathematics}, 17:\penalty0 527--566, 2017.

\bibitem[Bergou et~al.(2020)Bergou, Gorbunov, and Richt{\'a}rik]{bergou2020stochastic}
El~Houcine Bergou, Eduard Gorbunov, and Peter Richt{\'a}rik.
\newblock Stochastic three points method for unconstrained smooth minimization.
\newblock \emph{SIAM Journal on Optimization}, 30\penalty0 (4):\penalty0 2726--2749, 2020.

\bibitem[Ghadimi and Lan(2013)]{ghadimi2013stochastic}
Saeed Ghadimi and Guanghui Lan.
\newblock Stochastic first-and zeroth-order methods for nonconvex stochastic programming.
\newblock \emph{SIAM Journal on Optimization}, 23\penalty0 (4):\penalty0 2341--2368, 2013.

\bibitem[Chen et~al.(2020)Chen, Orvieto, and Lucchi]{chen2020accelerated}
Yuwen Chen, Antonio Orvieto, and Aurelien Lucchi.
\newblock An accelerated dfo algorithm for finite-sum convex functions.
\newblock \emph{arXiv preprint arXiv:2007.03311}, 2020.

\bibitem[Lewis et~al.(2000)Lewis, Torczon, and Trosset]{lewis2000direct}
Robert~Michael Lewis, Virginia Torczon, and Michael~W Trosset.
\newblock Direct search methods: then and now.
\newblock \emph{Journal of computational and Applied Mathematics}, 124\penalty0 (1-2):\penalty0 191--207, 2000.

\bibitem[Conn et~al.(2009)Conn, Scheinberg, and Vicente]{conn2009introduction}
Andrew~R Conn, Katya Scheinberg, and Luis~N Vicente.
\newblock \emph{Introduction to derivative-free optimization}.
\newblock SIAM, 2009.

\bibitem[Vicente(2013)]{vicente2013worst}
Lu{\'i}s~Nunes Vicente.
\newblock Worst case complexity of direct search.
\newblock \emph{EURO Journal on Computational Optimization}, 1\penalty0 (1-2):\penalty0 143--153, 2013.

\bibitem[Zelikman et~al.(2023)Zelikman, Huang, Liang, Haber, and Goodman]{zelikman2023just}
Eric Zelikman, Qian Huang, Percy Liang, Nick Haber, and Noah~D Goodman.
\newblock Just one byte (per gradient): A note on low-bandwidth decentralized language model finetuning using shared randomness.
\newblock \emph{arXiv preprint arXiv:2306.10015}, 2023.

\bibitem[Salimans et~al.(2017)Salimans, Ho, Chen, Sidor, and Sutskever]{salimans2017evolution}
Tim Salimans, Jonathan Ho, Xi~Chen, Szymon Sidor, and Ilya Sutskever.
\newblock Evolution strategies as a scalable alternative to reinforcement learning.
\newblock \emph{arXiv preprint arXiv:1703.03864}, 2017.

\bibitem[Mania et~al.(2018)Mania, Guy, and Recht]{mania2018simple}
Horia Mania, Aurelia Guy, and Benjamin Recht.
\newblock Simple random search of static linear policies is competitive for reinforcement learning.
\newblock \emph{Advances in Neural Information Processing Systems}, 31, 2018.

\bibitem[Zhang et~al.(2019)Zhang, He, Sra, and Jadbabaie]{zhang2019gradient}
Jingzhao Zhang, Tianxing He, Suvrit Sra, and Ali Jadbabaie.
\newblock Why gradient clipping accelerates training:a theoretical justification for adaptivity.
\newblock \emph{arXiv preprint arXiv:1905.11881}, 2019.

\bibitem[Zhang et~al.(2020)Zhang, Jin, Fang, and Wang]{zhang2020improved}
Bohang Zhang, Jikai Jin, Cong Fang, and Liwei Wang.
\newblock Improved analysis of clipping algorithms for non-convex optimization.
\newblock \emph{Advances in Neural Information Processing Systems}, 33:\penalty0 15511--15521, 2020.

\bibitem[Danilova et~al.(2022)Danilova, Dvurechensky, Gasnikov, Gorbunov, Guminov, Kamzolov, and Shibaev]{danilova2022recent}
Marina Danilova, Pavel Dvurechensky, Alexander Gasnikov, Eduard Gorbunov, Sergey Guminov, Dmitry Kamzolov, and Innokentiy Shibaev.
\newblock Recent theoretical advances in non-convex optimization.
\newblock In \emph{High-Dimensional Optimization and Probability: With a View Towards Data Science}, pages 79--163. Springer, 2022.

\bibitem[Wang et~al.(2023)Wang, Zhang, Ma, and Chen]{wang2023convergence}
Bohan Wang, Huishuai Zhang, Zhiming Ma, and Wei Chen.
\newblock Convergence of adagrad for non-convex objectives: Simple proofs and relaxed assumptions.
\newblock In \emph{The Thirty Sixth Annual Conference on Learning Theory}, pages 161--190. PMLR, 2023.

\bibitem[Li et~al.(2023)Li, Jadbabaie, and Rakhlin]{li2023convergence}
Haochuan Li, Ali Jadbabaie, and Alexander Rakhlin.
\newblock Convergence of adam under relaxed assumptions.
\newblock \emph{arXiv preprint arXiv:2304.13972}, 2023.

\bibitem[Kingma and Ba(2014)]{kingma2014adam}
Diederik~P Kingma and Jimmy Ba.
\newblock Adam: A method for stochastic optimization.
\newblock \emph{arXiv preprint arXiv:1412.6980}, 2014.

\bibitem[Wang et~al.(2022)Wang, Zhang, Zhang, Meng, Ma, Liu, and Chen]{wang2022provable}
Bohan Wang, Yushun Zhang, Huishuai Zhang, Qi~Meng, Zhi-Ming Ma, Tie-Yan Liu, and Wei Chen.
\newblock Provable adaptivity in adam.
\newblock \emph{arXiv preprint arXiv:2208.09900}, 2022.

\bibitem[Li et~al.(2024)Li, Rakhlin, and Jadbabaie]{li2024convergence}
Haochuan Li, Alexander Rakhlin, and Ali Jadbabaie.
\newblock Convergence of adam under relaxed assumptions.
\newblock \emph{Advances in Neural Information Processing Systems}, 36, 2024.

\bibitem[Shamir(2017)]{shamir2017optimal}
Ohad Shamir.
\newblock An optimal algorithm for bandit and zero-order convex optimization with two-point feedback.
\newblock \emph{The Journal of Machine Learning Research}, 18\penalty0 (1):\penalty0 1703--1713, 2017.

\bibitem[Jamieson et~al.(2012)Jamieson, Nowak, and Recht]{jamieson2012query}
Kevin~G Jamieson, Robert Nowak, and Ben Recht.
\newblock Query complexity of derivative-free optimization.
\newblock \emph{Advances in Neural Information Processing Systems}, 25, 2012.

\bibitem[Agarwal et~al.(2009)Agarwal, Wainwright, Bartlett, and Ravikumar]{agarwal2009information}
Alekh Agarwal, Martin~J Wainwright, Peter Bartlett, and Pradeep Ravikumar.
\newblock Information-theoretic lower bounds on the oracle complexity of convex optimization.
\newblock \emph{Advances in Neural Information Processing Systems}, 22, 2009.

\bibitem[Raginsky and Rakhlin(2011)]{raginsky2011information}
Maxim Raginsky and Alexander Rakhlin.
\newblock Information-based complexity, feedback and dynamics in convex programming.
\newblock \emph{IEEE Transactions on Information Theory}, 57\penalty0 (10):\penalty0 7036--7056, 2011.

\bibitem[Duchi et~al.(2015)Duchi, Jordan, Wainwright, and Wibisono]{duchi2015optimal}
John~C Duchi, Michael~I Jordan, Martin~J Wainwright, and Andre Wibisono.
\newblock Optimal rates for zero-order convex optimization: The power of two function evaluations.
\newblock \emph{IEEE Transactions on Information Theory}, 61\penalty0 (5):\penalty0 2788--2806, 2015.

\bibitem[Gao and Sener(2022)]{gao2022generalizing}
Katelyn Gao and Ozan Sener.
\newblock Generalizing gaussian smoothing for random search.
\newblock In \emph{International Conference on Machine Learning}, pages 7077--7101. PMLR, 2022.

\bibitem[Choromanski et~al.(2018)Choromanski, Rowland, Sindhwani, Turner, and Weller]{choromanski2018structured}
Krzysztof Choromanski, Mark Rowland, Vikas Sindhwani, Richard Turner, and Adrian Weller.
\newblock Structured evolution with compact architectures for scalable policy optimization.
\newblock In \emph{International Conference on Machine Learning}, pages 970--978. PMLR, 2018.

\bibitem[Maheswaranathan et~al.(2019)Maheswaranathan, Metz, Tucker, Choi, and Sohl-Dickstein]{maheswaranathan2019guided}
Niru Maheswaranathan, Luke Metz, George Tucker, Dami Choi, and Jascha Sohl-Dickstein.
\newblock Guided evolutionary strategies: Augmenting random search with surrogate gradients.
\newblock In \emph{International Conference on Machine Learning}, pages 4264--4273. PMLR, 2019.

\bibitem[Cai et~al.(2022)Cai, Mckenzie, Yin, and Zhang]{cai2022zeroth}
HanQin Cai, Daniel Mckenzie, Wotao Yin, and Zhenliang Zhang.
\newblock Zeroth-order regularized optimization (zoro): Approximately sparse gradients and adaptive sampling.
\newblock \emph{SIAM Journal on Optimization}, 32\penalty0 (2):\penalty0 687--714, 2022.

\bibitem[Wang et~al.(2018{\natexlab{a}})Wang, Du, Balakrishnan, and Singh]{wang2018stochastic}
Yining Wang, Simon Du, Sivaraman Balakrishnan, and Aarti Singh.
\newblock Stochastic zeroth-order optimization in high dimensions.
\newblock In \emph{International conference on artificial intelligence and statistics}, pages 1356--1365. PMLR, 2018{\natexlab{a}}.

\bibitem[Sener and Koltun(2020)]{sener2020learning}
Ozan Sener and Vladlen Koltun.
\newblock Learning to guide random search.
\newblock \emph{arXiv preprint arXiv:2004.12214}, 2020.

\bibitem[Lucchi et~al.(2021)Lucchi, Orvieto, and Solomou]{lucchi2021second}
Aurelien Lucchi, Antonio Orvieto, and Adamos Solomou.
\newblock On the second-order convergence properties of random search methods.
\newblock \emph{Advances in Neural Information Processing Systems}, 34:\penalty0 25633--25645, 2021.

\bibitem[Zhang and Gu(2022)]{zhang2022faster}
Hualin Zhang and Bin Gu.
\newblock Faster gradient-free methods for escaping saddle points.
\newblock In \emph{The Eleventh International Conference on Learning Representations}, 2022.

\bibitem[Ren et~al.(2023)Ren, Tang, and Li]{ren2023escaping}
Zhaolin Ren, Yujie Tang, and Na~Li.
\newblock Escaping saddle points in zeroth-order optimization: the power of two-point estimators.
\newblock \emph{arXiv preprint arXiv:2209.13555}, 2023.

\bibitem[Bergstra and Bengio(2012)]{bergstra2012random}
James Bergstra and Yoshua Bengio.
\newblock Random search for hyper-parameter optimization.
\newblock \emph{Journal of machine learning research}, 13\penalty0 (2), 2012.

\bibitem[Yang and Shami(2020)]{yang2020hyperparameter}
Li~Yang and Abdallah Shami.
\newblock On hyperparameter optimization of machine learning algorithms: Theory and practice.
\newblock \emph{Neurocomputing}, 415:\penalty0 295--316, 2020.

\bibitem[Ilyas et~al.(2018)Ilyas, Engstrom, Athalye, and Lin]{ilyas2018black}
Andrew Ilyas, Logan Engstrom, Anish Athalye, and Jessy Lin.
\newblock Black-box adversarial attacks with limited queries and information.
\newblock In \emph{International conference on machine learning}, pages 2137--2146. PMLR, 2018.

\bibitem[Guo et~al.(2019)Guo, Gardner, You, Wilson, and Weinberger]{guo2019simple}
Chuan Guo, Jacob Gardner, Yurong You, Andrew~Gordon Wilson, and Kilian Weinberger.
\newblock Simple black-box adversarial attacks.
\newblock In \emph{International Conference on Machine Learning}, pages 2484--2493. PMLR, 2019.

\bibitem[Liu et~al.(2018)Liu, Chen, Chen, and Hong]{liu2018signsgd}
Sijia Liu, Pin-Yu Chen, Xiangyi Chen, and Mingyi Hong.
\newblock signsgd via zeroth-order oracle.
\newblock In \emph{International Conference on Learning Representations}, 2018.

\bibitem[Deng et~al.(2022)Deng, Wang, Hsieh, Wang, Guo, Shu, Song, Xing, and Hu]{deng2022rlprompt}
Mingkai Deng, Jianyu Wang, Cheng-Ping Hsieh, Yihan Wang, Han Guo, Tianmin Shu, Meng Song, Eric~P Xing, and Zhiting Hu.
\newblock Rlprompt: Optimizing discrete text prompts with reinforcement learning.
\newblock \emph{arXiv preprint arXiv:2205.12548}, 2022.

\bibitem[Khintchine(1923)]{khintchine1923dyadische}
Aleksandr Khintchine.
\newblock {\"U}ber dyadische br{\"u}che.
\newblock \emph{Mathematische Zeitschrift}, 18\penalty0 (1):\penalty0 109--116, 1923.

\bibitem[Kunstner et~al.(2023)Kunstner, Chen, Lavington, and Schmidt]{kunstner2023noise}
Frederik Kunstner, Jacques Chen, Jonathan~Wilder Lavington, and Mark Schmidt.
\newblock Noise is not the main factor behind the gap between sgd and adam on transformers, but sign descent might be.
\newblock \emph{arXiv preprint arXiv:2304.13960}, 2023.

\bibitem[Huang et~al.(2017)Huang, Liu, Van Der~Maaten, and Weinberger]{huang2017densely}
Gao Huang, Zhuang Liu, Laurens Van Der~Maaten, and Kilian~Q Weinberger.
\newblock Densely connected convolutional networks.
\newblock In \emph{Proceedings of the IEEE conference on computer vision and pattern recognition}, pages 4700--4708, 2017.

\bibitem[He et~al.(2016)He, Zhang, Ren, and Sun]{he2016deep}
Kaiming He, Xiangyu Zhang, Shaoqing Ren, and Jian Sun.
\newblock Deep residual learning for image recognition.
\newblock In \emph{Proceedings of the IEEE conference on computer vision and pattern recognition}, pages 770--778, 2016.

\bibitem[Dosovitskiy et~al.(2021)Dosovitskiy, Beyer, Kolesnikov, Weissenborn, Zhai, Unterthiner, Dehghani, Minderer, Heigold, Gelly, Uszkoreit, and Houlsby]{dosovitskiy2020vit}
Alexey Dosovitskiy, Lucas Beyer, Alexander Kolesnikov, Dirk Weissenborn, Xiaohua Zhai, Thomas Unterthiner, Mostafa Dehghani, Matthias Minderer, Georg Heigold, Sylvain Gelly, Jakob Uszkoreit, and Neil Houlsby.
\newblock An image is worth 16x16 words: Transformers for image recognition at scale.
\newblock \emph{ICLR}, 2021.

\bibitem[Simonyan and Zisserman(2014)]{simonyan2014very}
Karen Simonyan and Andrew Zisserman.
\newblock Very deep convolutional networks for large-scale image recognition.
\newblock \emph{arXiv preprint arXiv:1409.1556}, 2014.

\bibitem[Deng(2012)]{deng2012mnist}
Li~Deng.
\newblock The mnist database of handwritten digit images for machine learning research.
\newblock \emph{IEEE Signal Processing Magazine}, 29\penalty0 (6):\penalty0 141--142, 2012.

\bibitem[Krizhevsky and Hinton(2009)]{krizhevsky2009learning}
Alex Krizhevsky and Geoffrey Hinton.
\newblock Learning multiple layers of features from tiny images.
\newblock Technical report, University of Toronto, Toronto, Ontario, 2009.

\bibitem[Radford et~al.(2019)Radford, Wu, Child, Luan, Amodei, Sutskever, et~al.]{radford2019language}
Alec Radford, Jeffrey Wu, Rewon Child, David Luan, Dario Amodei, Ilya Sutskever, et~al.
\newblock Language models are unsupervised multitask learners.
\newblock \emph{OpenAI blog}, 1\penalty0 (8):\penalty0 9, 2019.

\bibitem[Wang et~al.(2018{\natexlab{b}})Wang, Singh, Michael, Hill, Levy, and Bowman]{wang2018glue}
Alex Wang, Amanpreet Singh, Julian Michael, Felix Hill, Omer Levy, and Samuel~R Bowman.
\newblock Glue: A multi-task benchmark and analysis platform for natural language understanding.
\newblock \emph{arXiv preprint arXiv:1804.07461}, 2018{\natexlab{b}}.

\end{thebibliography}
\appendix

\section{Technical lemmas}
\label{append_section_tech}

\begin{lemma}
\label{lemma_l0l1_descent}
(\citep{zhang2020improved} Descent Inequality) Suppose objective function $f(\cdot)$ satisfies Assumption~\ref{ass_l0l1}, and $c>0$ be a constant. For any $\rvx_{k}$ and $\rvx_{k+1}$, as long as $||\rvx_{k} - \rvx_{k+1}||\leq \frac{c}{L_{1}}$, we have
\begin{align}
\label[ineq]{l0l1_descent_inequality}
f(\rvx_{k+1}) \leq f(\rvx_{k}) + (\rvx_{k+1} - \rvx_{k})^{T}\nabla f(\rvx_{k}) + \frac{AL_{0} + BL_{1}||\nabla f(\rvx_{k})||}{2}||\rvx_{k+1} - \rvx_{k}||^{2}
\end{align}
where $A = 1+e^{c}-\frac{e^{c}-1}{c}, B = \frac{e^{c} - 1}{c}$. Note that $A$ and $B$ are monotonically increasing functions w.r.t. $c>0$.
\end{lemma}

\begin{lemma}
\label{append_lemma_sign-g_sign-s}
(Lemma~\ref{lemma_sign-g_sign-s}) For all $\rvg\in\mathbb{R}^{d}$, and random vector $\rvs \sim \mathcal{R}$ where $\mathcal{R}$ is the Rademacher distribution, i.e., each element $\ervs \sim \{+1, -1 \}$ with equal chances and $\mathbb{E}_{\rvs\sim \mathcal{R}}||\rvs||_{2}^{2}=d$, then
$
\mathbb{E}_{\rvs\sim \mathcal{R}}|\langle\rvg, \rvs\rangle| \geq \frac{1}{\sqrt{2}}||\rvg||_{2}.
$
\end{lemma}
\begin{proof}
\begin{align}
\label{append_proof_eq_expectation}
\begin{split}
|\langle\rvg, \rvs\rangle| = |\sum_{i=1}^{d}\ervg_{i}\ervs_{i}|
\end{split}
\end{align}
According to Khintchine inequality~\cite{khintchine1923dyadische}, i.e., 
\begin{align*}
A_{p}(\sum_{i=1}^{d}|\ervg_{i}|^{2})^{\frac{1}{2}} \leq (\mathbb{E}|\sum_{i=1}^{d}\ervg_{i}\ervs_{i}|^{p})^{\frac{1}{p}} \leq  B_{p}(\sum_{i=1}^{d}|\ervg_{i}|^{2})^{\frac{1}{2}}
\end{align*}
where
\begin{align*}
A_{p} =  \left\{ \begin{array}{rl}
2^{\frac{1}{2} - \frac{1}{p}} & 0<p<p_{0} \\ 2^{\frac{1}{2}}(\Gamma((p+1)/2)/\sqrt{\pi})^{\frac{1}{p}} & p_{0}<p<2 \\
1 & 2\leq p < \infty.
\end{array}\right.
\end{align*}
\begin{align*}
B_{p} =  \left\{ \begin{array}{rcl}
1 &0<p\leq2 \\ 
2^{\frac{1}{2}}(\Gamma((p+1)/2)/\sqrt{\pi})^{\frac{1}{p}} &2<p< \infty.
\end{array}\right.
\end{align*}
where $p_{0}\approx 1.847$ and $\Gamma$ is the Gamma function, we have
\begin{align*}
\frac{1}{\sqrt{2}}||\rvg||_{2} \leq \mathbb{E}|\sum_{i=1}^{d}\ervg_{i}\ervs_{i}| \leq ||\rvg||_{2},
\end{align*}
Combined with \eqref{append_proof_eq_expectation}, we have 
\begin{align*}
\frac{1}{\sqrt{2}}||\rvg||_{2} \leq \mathbb{E}_{\rvs\sim \mathcal{R}}|\langle\rvg, \rvs\rangle| \leq ||\rvg||_{2}.
\end{align*}
This completes the proof.
\end{proof}

\section{Convergence analysis under the general smoothness assumption}
\label{appen_cor_stsp}

\subsection{Progressive bound of S2P}
\begin{lemma}
\label{append_lemma_progress}
(Lemma~\ref{lemma_progress}) (Progressive bound) Suppose objective function $f(\cdot)$ satisfies Assumption~\ref{ass_lsmooth_gd} and $||\nabla f(\rvx_{k})||_{2}\geq \epsilon_{g}$. If we run algorithm~\ref{alg_algorithm_step_size} with step size $\alpha = \frac{\sqrt{2}\epsilon_{g}}{2Ld}$, we have following progressive bound
$
\mathbb{E}[f(\rvx_{k+1}) - f(\rvx_{k})|\rvx_{k}] \leq -\Omega(\frac{\epsilon_{g}^{2}}{Ld}),
$ where $\mathbb{E}[\cdot|\rvx_{k}]$ denotes the conditional expectation w.r.t. $\rvx_{k}$.
\end{lemma}
\begin{proof}
Using $L$-gradient Lipschitz, we have (descent lemma)
\begin{align*}
    \mathbb{E}&[f(\rvx_{k+1}) - f(\rvx_{k})|\rvx_{k}] \\
    &\leq  \mathbb{E}[\nabla f(\rvx_{k})^{T}(\rvx_{k+1} - \rvx_{k})|\rvx_{k}] + \frac{L}{2}\mathbb{E}[||\rvx_{k+1} - \rvx_{k}||^{2}] \\
    &= -\alpha\mathbb{E}|\nabla f(\rvx_{k})^{T}\rvs_{k}| + \frac{L\alpha^{2}}{2}\mathbb{E}||\rvs_{k}||_{2}^{2} \;\;\;\;\text{Take updating step}\\
    &= -\alpha\mathbb{E}|\nabla f(\rvx_{k})^{T}\rvs_{k}| + \frac{L\alpha^{2}d}{2}
\end{align*}
Lemma~\ref{append_proof_eq_expectation} shows that 
$\mathbb{E}_{\rvs_{k}\sim \mathcal{R}}|\nabla f(\rvx_{k})^{T}\rvs_{k}| \geq \frac{1}{\sqrt{2}}||\nabla f(\rvx_{k})||_{2}$, then 
\begin{align*}
    \mathbb{E}[f(\rvx_{k+1}) - f(\rvx_{k})|\rvx_{k}] &\leq -\frac{\alpha}{\sqrt{2}}||\nabla f(\rvx_{k})||_{2} + \frac{L\alpha^{2}d}{2} \\
    &\leq -\frac{\alpha}{\sqrt{2}}\epsilon_{g} + \frac{L\alpha^{2}d}{2}
\end{align*}
To guarantee convergence, $\alpha\sim [0, \frac{\sqrt{2}\epsilon_{g}}{Ld}]$, then suppose $\alpha = \frac{\sqrt{2}\epsilon_{g}}{2Ld}$, we have
$
\mathbb{E}[f(\rvx_{k+1}) - f(\rvx_{k})|\rvx_{k}] \leq -\frac{\epsilon_{g}^{2}}{4Ld}
$ which completes the proof.
\end{proof}

\subsection{Query complexity of S2P}
\begin{theorem}
\label{append_theorem_stationary}
(Theorem~\ref{theorem_stationary}) (Query complexity) Suppose objective function $f(\cdot)$ satisfies Assumption~\ref{ass_lsmooth_gd}. If we run algorithm~\ref{alg_algorithm_step_size} with step size strategy options 1 or 2, the algorithm returns in expectation an $\epsilon$-first-order stationary point in $\mathcal{O}(\frac{d}{\epsilon^{2}})$ function evaluations.
\end{theorem}
\begin{proof}
Using $L$-gradient Lipschitz, we have (descent lemma)
\begin{align}
    \mathbb{E}[f(\rvx_{k+1})|\rvx_{k}] &\leq f(\rvx_{k}) + \mathbb{E}[\nabla f(\rvx_{k})^{T}(\rvx_{k+1} - \rvx_{k})|\rvx_{k}] + \frac{L}{2}\mathbb{E}[||\rvx_{k+1} - \rvx_{k}||^{2}] \nonumber\\
    &= f(\rvx_{k}) -\alpha\mathbb{E}|\nabla f(\rvx_{k})^{T}\rvs_{k}| + \frac{L\alpha^{2}}{2}\mathbb{E}||\rvs_{k}||_{2}^{2} \nonumber\\
    &= f(\rvx_{k}) -\alpha\mathbb{E}|\nabla f(\rvx_{k})^{T}\rvs_{k}| + \frac{L\alpha^{2}d}{2} \label[ineq]{append_l0_inequality_00}
\end{align}

\textbf{Option 1. Stationary step size} 

Lemma~\ref{append_proof_eq_expectation} shows that 
$\mathbb{E}_{\rvs_{k}\sim \mathcal{R}}|\nabla f(\rvx_{k})^{T}\rvs_{k}| \geq \frac{1}{\sqrt{2}}||\nabla f(\rvx_{k})||_{2}$, then~\cref{append_l0_inequality_00} can be reformulated as
\begin{align*}
\mathbb{E}[f(\rvx_{k+1}) |\rvx_{k}] &\leq f(\rvx_{k}) -\frac{\alpha}{\sqrt{2}}||\nabla f(\rvx_{k})||_{2} + \frac{L\alpha^{2}d}{2}
\end{align*}

Taking expectations in the above inequality w.r.t. $\rvs_{k}$ conditional on $\rvx_{k}$, and denoting $\theta_{k} = \mathbb{E}[f(\rvx_{k+1})]$ and $g_{k} = \mathbb{E}[||\nabla f(\rvx_{k})||_{2}]$, we have
\begin{align*}
    \theta_{k+1} &\leq \theta_{k} -\frac{\alpha g_{k}}{\sqrt{2}} + \frac{L\alpha^{2}d}{2}\\
    g_{k} &\leq \sqrt{2}(\frac{\theta_{k} - \theta_{k+1}}{\alpha} + \frac{L\alpha d}{2}) \\
    \sum_{k = 0}^{K}g_{k} &\leq \sqrt{2}(\frac{\theta_{0} - \theta_{k+1}}{\alpha} + \frac{KL\alpha d}{4})
\end{align*}
We can conclude that there exists an iteration $j\sim [0, K]$ such that
\begin{align*}
    g_{j} &\leq \sqrt{2}(\frac{\theta_{0} - \theta_{k+1}}{\alpha K} + \frac{L\alpha d}{2}) \\
    g_{j} &\leq \sqrt{2}(\frac{(f(\rvx_{0}) - f^{\star})\sqrt{Kd}}{\alpha_{0} K} + \frac{L\alpha_{0}\sqrt{d}}{2\sqrt{K}}) \;\;(\alpha = \frac{\alpha_{0}}{\sqrt{Kd}})\\
    g_{j} &\leq \frac{\sqrt{2d}}{\sqrt{K}}(\frac{(f(\rvx_{0}) - f^{\star})}{\alpha_{0}} + \frac{L\alpha_{0}}{2})
\end{align*}
Then let $\frac{\sqrt{2d}}{\sqrt{K}}(\frac{(f(\rvx_{0}) - f^{\star})}{\alpha_{0}} + \frac{L\alpha_{0}}{2}) \leq \epsilon$, we have
\begin{align*}
    K \geq \frac{2d}{\epsilon^{2}}(\frac{(f(\rvx_{0}) - f^{\star})}{\alpha_{0}} + \frac{L\alpha_{0}}{2})^{2},
\end{align*}
, which completes the proof for option 1.

\textbf{Option 2. Dynamic step size} 

Taking expectations in the above~\cref{append_l0_inequality_00} w.r.t. $\rvs_{k}$ conditional on $\rvx_{k}$, and denoting $\theta_{k} = \mathbb{E}[f(\rvx_{k+1})]$, we have
\begin{align}
    \theta_{k+1} \leq  \theta_{k} -\alpha|\nabla f(\rvx_{k})^{T}\rvs_{k}| + \frac{L\alpha^{2}d}{2} \label[ineq]{append_descent_query}
\end{align}
We know that the best $\alpha_{k}^{opt} = \frac{|\nabla f(\rvx_{k})^{T}\rvs_{k}|}{Ld}$, and we can approximate the best step size with $\alpha_{k} =  \frac{|f(\rvx + \rho\rvs_{k}) - f(\rvx - \rho\rvs_{k})|}{2\rho Ld}$ (or $\alpha_{k} = \alpha_{0} \frac{|f(\rvx + \rho\rvs_{k}) - f(\rvx - \rho\rvs_{k})|}{2\rho}$ where $\alpha_{0} = \frac{1}{Ld}$) where $\rho$ is a scalar.

Before continuing working on the~\cref{append_descent_query}, we estimate the error between the best step size and the approximated step size, $|\delta_{k}| := |\alpha_{k} - \alpha_{k}^{opt}|$, firstly.
\begin{align}
|\delta_{k}| &= \frac{1}{2\rho Ld}\big||f(\rvx + \rho\rvs_{k}) - f(\rvx - \rho\rvs_{k})| - 2\rho|\nabla f(\rvx_{k})^{T}\rvs_{k}|\big|\nonumber\\
&\leq \frac{1}{2\rho Ld}|f(\rvx + \rho\rvs_{k}) - f(\rvx - \rho\rvs_{k}) - 2\rho\nabla f(\rvx_{k})^{T}\rvs_{k}| \label[ineq]{append_eq_revsrse_tri}\\
&= \frac{1}{2\rho Ld}|(f(\rvx + \rho\rvs_{k}) - f(\rvx) -\rho\nabla f(\rvx_{k})^{T}\rvs_{k}) - (f(\rvx - \rho\rvs_{k}) - f(\rvx) + \rho\nabla f(\rvx_{k})^{T}\rvs_{k})| \nonumber\\
&\leq \frac{1}{2\rho Ld}(\frac{L}{2}\rho^{2} ||\rvs_{k}||^{2} + \frac{L}{2}\rho^{2} ||\rvs_{k}||^{2}) \label[ineq]{append_eq_l_smoothness_equ}\\
&\leq \frac{\rho}{2} \label[ineq]{append_eq_lr_error}
\end{align}
Note that~\cref{append_eq_revsrse_tri} applied reverse triangle inequality and~\cref{append_eq_l_smoothness_equ} applied the equivalent definitions of $L$-smooth function $|f(\rvx + \rho\rvs_{k}) - f(\rvx) -\rho\nabla f(\rvx_{k})^{T}\rvs_{k}|\leq \frac{L}{2}||\rho\rvs_{k}||^{2}$.

Suppose we do take $\alpha_{k} = \frac{|f(\rvx + \rho\rvs_{k}) - f(\rvx - \rho\rvs_{k})|}{2\rho Ld}$ and substitute $\alpha_{k} = \alpha_{k}^{opt} + \delta_{k}$, \cref{append_descent_query} can be reformulated as
\begin{align}
    \theta_{k+1} &\leq  \theta_{k} - (\alpha_{k}^{opt} + \delta_{k})|\nabla f(\rvx_{k})^{T}\rvs_{k}| + \frac{L(\alpha_{k}^{opt} + \delta_{k})^{2}d}{2}\nonumber\\
    &=  \theta_{k} -  \frac{|\nabla f(\rvx_{k})^{T}\rvs_{k}|^{2}}{Ld} - \delta_{k}|\nabla f(\rvx_{k})^{T}\rvs_{k}|+  \frac{|\nabla f(\rvx_{k})^{T}\rvs_{k}|^{2}}{2Ld} + \delta_{k}|\nabla f(\rvx_{k})^{T}\rvs_{k}| + \frac{Ld \delta_{k}^{2}}{2}\nonumber\\
    &= \theta_{k} -\frac{|\nabla f(\rvx_{k})^{T}\rvs_{k}|^{2}}{2Ld} + \frac{Ld \delta_{k}^{2}}{2}\nonumber\\
    &\leq \theta_{k} -\frac{|\nabla f(\rvx_{k})^{T}\rvs_{k}|^{2}}{2Ld} + \frac{Ld \rho^{2}}{8}\;\;\text{Apply~\cref{append_eq_lr_error}}\nonumber\\
    &\leq \theta_{k} -\frac{||\nabla f(\rvx_{k})||^{2}}{4Ld} + \frac{Ld \rho^{2}}{8}\label[ineq]{append_rho_descent}\;\;\;\;\text{Apply Lemma~\ref{append_proof_eq_expectation}}
\end{align}
Note that it actually put requirement on $\rho$ to guarantee convergence, i.e., for $\rho_{k}$ in each iterations, we need $0<\rho\leq \frac{\sqrt{2}||\nabla f(\rvx_{k})||}{Ld}$.

Continually,~\cref{append_rho_descent} further can be re-formulated as
\begin{align*}
||\nabla f(\rvx_{k})||^{2} &\leq 4Ld(\theta_{k} - \theta_{k+1}) + \frac{\rho^{2}}{2}\\
\sum_{k=0}^{K} ||f(\rvx_{k})||^{2} &\leq 4Ld(\theta_{0} - \theta_{k+1}) + \frac{K\rho^{2}}{2}
\end{align*}
We can conclude that there exists an iteration $j\sim [0, K]$ such that
\begin{align*}
||f(\rvx_{j})||^{2} \leq \frac{4Ld(\theta_{0} - \theta_{k+1})}{K} + \frac{\rho^{2}}{2}\leq \frac{4Ld(f(\rvx_{0})  - f^{\star})}{K} + \frac{\rho^{2}}{2}
\end{align*}
which further concludes that we need 
\begin{align}
K \geq \frac{4Ld (f(\rvx_{0}) - f^{\star})}{\epsilon^{2} - \frac{\rho^{2}}{2}},
\end{align} 
iterations to reach $\epsilon$-first-order stationary point ($||f(\rvx_{j})|| \leq \epsilon$).

Meanwhile, we require that $0<\rho_{k} \leq \frac{\sqrt{2}||\nabla f(\rvx_{k})||}{Ld}$ for $\rho_{k}$ in each iterations, and it can be set to a small value universally.
E.g., $0<\rho \leq \frac{\sqrt{2}\epsilon}{Ld}$, then we have
$K \geq \frac{4Ld (f(\rvx_{0}) - f^{\star})}{\epsilon^{2}(1- \frac{1}{L^{2}d^{2}})}$.

Then, we can safely conclude that the algorithm returns in expectation an $\epsilon$-first-order stationary point in $\mathcal{O}(\frac{d}{\epsilon^{2}})$ function evaluations, which completes the proof for option 2.
\end{proof}

\section{Convergence analysis under the relaxed smoothness assumption}
\label{append_relaxed}

\subsection{Progressive bound of S2P}
\label{append_subsection_relaxed_progressive}
\begin{lemma}
\label{append_lemma_l0l1_progress}
(Lemma~\ref{lemma_l0l1_progress}) (Progressive bound) Suppose objective function $f(\cdot)$ satisfies Assumption~\ref{ass_l0l1} and $||\nabla f(\rvx_{k})||_{2}\geq \epsilon_{g}$. If we run algorithm~\ref{alg_algorithm_step_size} with step size $\alpha = \frac{\sqrt{2}\epsilon_{g}}{2(AL_{0} + BL_{1}\epsilon_{g})d}$, we have following progressive bound
$
\mathbb{E}[f(\rvx_{k+1}) - f(\rvx_{k})|\rvx_{k}] \leq -\Omega(\frac{\epsilon_{g}^{2}}{(AL_{0} + BL_{1}\epsilon_{g})d}),
$ where $\mathbb{E}[\cdot|\rvx_{k}]$ denotes the conditional expectation w.r.t. $\rvx_{k}$, and constants $A = 1.01, B = 1.01$.
\end{lemma}
\begin{proof}
Give the decent lemma~\cref{l0l1_descent_inequality}, we have
\begin{align}
\mathbb{E}[f(\rvx_{k+1})] &\leq f(\rvx_{k}) - \alpha\mathbb{E}[\rvg_{k}^{T}\nabla f(\rvx_{k})] + \frac{AL_{0} + BL_{1}||\nabla f(\rvx_{k})||}{2}\mathbb{E}[\alpha^{2}||\rvg_{k}||^{2}]\nonumber\\
&= f(\rvx_{k}) - \alpha\mathbb{E}[|\rvs_{k}^{T}\nabla f(\rvx_{k})|]+ \frac{A L_{0} + BL_{1}||\nabla f(\rvx_{k})||}{2}\mathbb{E}[\alpha^{2}||\rvs_{k}||^{2}]\;\;\text{Updating step}\nonumber\\
&\leq f(\rvx_{k}) - \frac{\alpha}{\sqrt{2}}||\nabla f(\rvx_{k})|| + \alpha^{2}\frac{AL_{0} + BL_{1}||\nabla f(\rvx_{k})||}{2}d \;\;\text{Lemma~\ref{append_proof_eq_expectation}}\label{append_eq_l0l1_decent}
\end{align}
Suppose $||\nabla f(\rvx_{k})|| \geq \epsilon_{g}$, and to guarantee convergence $\alpha \in [0, \frac{\sqrt{2}\epsilon_{g}}{(AL_{0} + BL_{1}\epsilon_{g})d}]$. Let $\alpha = \frac{\sqrt{2}\epsilon_{g}}{2(AL_{0} + BL_{1}\epsilon_{g})d}$, we have
\begin{align*}
\mathbb{E}[f(\rvx_{k+1})] &\leq f(\rvx_{k}) - \frac{\epsilon_{g}^{2}}{4(AL_{0} + BL_{1}\epsilon_{g})d}.
\end{align*}
which completes the proof.

Note that for the specific value of $A$ and $B$, we have $A = 1+e^{c}-\frac{e^{c}-1}{c}, B = \frac{e^{c} - 1}{c}$ and $||\rvx_{k+1} - \rvx_{k}|| = ||\alpha\rvs_{k}|| = \frac{\sqrt{2}\epsilon_{g}}{2(AL_{0} + BL_{1}\epsilon_{g})\sqrt{d}}  \leq \frac{c}{L_{1}}\rightarrow c \geq \frac{\sqrt{2}L_{1}\epsilon_{g}}{2(AL_{0} + BL_{1}\epsilon_{g})\sqrt{d}} \rightarrow c \geq \frac{1}{\sqrt{2d}B} \rightarrow e^{c} \geq 1 + \frac{1}{\sqrt{2d}}$. 
It is easy to see that such $c$ exists, we can safely consider $A = 1.01, B= 1.01$ for simplicity (under large $d$) since $A$ and $B$ are expected to be small values.
\end{proof}

\subsection{Query complexity of S2P}
\label{append_subsection_relaxed_query_complexsity}
\begin{theorem}
\label{append_theorem_l0l1_stationary}
(Theorem~\ref{theorem_l0l1_query_complexsity}) (Query complexity) Suppose objective function $f(\cdot)$ satisfies Assumption~\ref{ass_l0l1}. If we run algorithm~\ref{alg_algorithm_step_size} with step size strategy options 3 or 4, the algorithm returns in expectation an $\epsilon$-first-order stationary point in $\mathcal{O}(\frac{d}{\epsilon^{2}})$ function evaluations.
\end{theorem}
\begin{proof}
Give the decent lemma~\cref{l0l1_descent_inequality}, we have
\begin{align}
\mathbb{E}[f(\rvx_{k+1})] &\leq f(\rvx_{k}) - \alpha\mathbb{E}[\rvg_{k}^{T}\nabla f(\rvx_{k})] + \frac{A L_{0} + BL_{1}||\nabla f(\rvx_{k})||}{2}\mathbb{E}[\alpha^{2}||\rvg_{k}||^{2}]\nonumber\\
&= f(\rvx_{k}) - \alpha\mathbb{E}[|\rvs_{k}^{T}\nabla f(\rvx_{k})|]+ \alpha^{2}\frac{A L_{0} + BL_{1}||\nabla f(\rvx_{k})||}{2}\mathbb{E}[||\rvs_{k}||^{2}]\;\;\text{Updating step}\label[ineq]{append_l0l1_inequality_00}
\end{align}

\textbf{Option 1. Stationary step size} 

Lemma~\ref{append_proof_eq_expectation} shows that 
$\mathbb{E}_{\rvs_{k}\sim \mathcal{R}}|\nabla f(\rvx_{k})^{T}\rvs_{k}| \geq \frac{1}{\sqrt{2}}||\nabla f(\rvx_{k})||_{2}$, then~\cref{append_l0l1_inequality_00} can be reformulated as
\begin{align*}
\mathbb{E}[f(\rvx_{k+1})] &\leq f(\rvx_{k}) - \frac{\alpha}{\sqrt{2}}||\nabla f(\rvx_{k})|| + \alpha^{2}\frac{A L_{0} + BL_{1}||\nabla f(\rvx_{k})||}{2}d
\end{align*}

Taking expectations in the above inequality w.r.t. $\rvs_{k}$ conditional on $\rvx_{k}$, and denoting $\theta_{k} = \mathbb{E}[f(\rvx_{k+1})]$ and $g_{k} = \mathbb{E}[||\nabla f(\rvx_{k})||]$, we have
\begin{align*}
    \theta_{k+1} \leq \theta_{k} -
    \frac{\alpha}{\sqrt{2}}g_{k} + \alpha^{2}\frac{A L_{0} + BL_{1}g_{k}}{2}d\\
    g_{k}(\frac{\sqrt{2}\alpha - B\alpha^{2}L_{1}d}{2}) \leq \theta_{k} - \theta_{k+1} + \frac{A\alpha^{2}L_{0}d}{2}\\
    g_{k}\leq \frac{2(\theta_{k} - \theta_{k+1})}{\sqrt{2}\alpha - B\alpha^{2}L_{1}d} + \frac{A\alpha^{2}L_{0}d}{\sqrt{2}\alpha - B\alpha^{2}L_{1}d}\\
    \sum_{k=0}^{K}g_{k} \leq \frac{2(\theta_{0} - \theta_{k+1})}{\sqrt{2}\alpha - B\alpha^{2}L_{1}d} + \frac{KA\alpha^{2}L_{0}d}{\sqrt{2}\alpha - B\alpha^{2}L_{1}d}
\end{align*}
We can conclude that there exists an iteration $j\sim [0, K]$ such that
\begin{align}
    g_{j} &\leq \frac{2(\theta_{0} - \theta_{K+1})}{(\sqrt{2}\alpha - B\alpha^{2}L_{1}d)K} + \frac{A\alpha^{2}L_{0}d}{\sqrt{2}\alpha - B\alpha^{2}L_{1}d}\nonumber\\
    &\leq \frac{2(f(\rvx_{0}) - f^{\star})}{(\sqrt{2}\alpha - B\alpha^{2}L_{1}d)K} + \frac{A\alpha^{2}L_{0}d}{\sqrt{2}\alpha - B\alpha^{2}L_{1}d} \label[ineq]{append_l0l1_grad_00}
\end{align}
Suppose $\alpha = \frac{\sqrt{2}}{BL_{1}\sqrt{dK}}$, \cref{append_l0l1_grad_00} can be reformulated as
\begin{align*}
    g_{j} &\leq \frac{B(f(\rvx_{0}) - f^{\star})L_{1}\sqrt{d}}{\sqrt{K} - \sqrt{d}} + \frac{AL_{0}\sqrt{d}}{BL_{1}(\sqrt{K} - \sqrt{d})}.
\end{align*}
Under this setting, we can see that the $g_{j}$ can be continually decreased with at least $K > d$, which further shows that it need
\begin{align*}
    K &\geq (\sqrt{d} + \frac{AL_{0}\sqrt{d} + BL_{1}(f(\rvx_{0}) - f^{\star})\sqrt{d}}{\epsilon})^{2}
\end{align*}
iterations to reach $\epsilon$-first-order stationary point.
Then, we can safely conclude that the algorithm returns in expectation an $\epsilon$-first-order stationary point in $\mathcal{O}(\frac{d}{\epsilon^{2}})$ function evaluations, which completes the proof for option 1.

Note that for the specific value of $A$ and $B$, we have $A = 1+e^{c}-\frac{e^{c}-1}{c}, B = \frac{e^{c} - 1}{c}$ and $||\rvx_{k+1} - \rvx_{k}|| = ||\alpha\rvs_{k}|| = \frac{\sqrt{2}}{BL_{1}\sqrt{K}}  \leq \frac{c}{L_{1}}\rightarrow c \geq \frac{\sqrt{2}}{B\sqrt{K}}\rightarrow e^{c} \geq 1 + \sqrt{\frac{2}{K}}$.
It is easy to see that such $c$ exists, we can safely consider $A = 1.01, B= 1.01$ for simplicity (under large $d$) since $A$ and $B$ are expected to be small values.

\textbf{Option 2. Dynamic step size} 

Taking expectations in the above~\cref{append_l0l1_inequality_00} w.r.t. $\rvs_{k}$ conditional on $\rvx_{k}$, and denoting $\theta_{k} = \mathbb{E}[f(\rvx_{k+1})]$, we have
\begin{align}
    \theta_{k+1} &\leq  \theta_{k} - \alpha|\rvs_{k}^{T}\nabla f(\rvx_{k})| + \alpha^{2}\frac{A L_{0} + BL_{1}||\nabla f(\rvx_{k})||}{2}d\nonumber\\
    &\leq  \theta_{k} - \alpha|\rvs_{k}^{T}\nabla f(\rvx_{k})| + \alpha^{2}\frac{A L_{0} + \sqrt{2}BL_{1}|\rvs_{k}^{T}\nabla f(\rvx_{k})|}{2}d.\label[ineq]{append_l0l1_dy_1}
\end{align}
It is easy to know that 
$\alpha^{opt}_{k} = \frac{|\rvs^{T}\nabla f(\rvx_{k})|}{(A L_{0} + \sqrt{2}BL_{1}|\rvs^{T}\nabla f(\rvx_{k})|)d}$. 
Let $|\gamma_{k}| = \frac{|f(\rvx_{k} + \rho\rvs_{k}) - f(\rvx_{k} - \rho\rvs_{k})|}{2\rho}$, and we approximate the best step size with  $\alpha_{k} = \frac{|\gamma_{k}|}{(A L_{0} + \sqrt{2}BL_{1}|\gamma_{k}|)d}$ and denote the approximation error as $|\delta_{k}|:= |\alpha_{k} - \alpha_{k}^{opt}|$. 

Before we continue working on the~\cref{append_l0l1_dy_1}, we derive the upper bound of $|\delta_{k}|$ for our following analysis.
Firstly, we denote $|\epsilon_{\rho}|:= \big||\rvs^{T}\nabla f(\rvx_{k})| - |\gamma_{k}|\big| =  \big||\rvs^{T}\nabla f(\rvx_{k})| - \frac{|f(\rvx_{k} + \rho\rvs_{k}) - f(\rvx_{k} - \rho\rvs_{k})|}{2\rho}\big| = \mathcal{O}(\rho^{2}d^{3/2}) $ (Taylor expansion). So that, we can define $|\epsilon_{\rho}| \leq \xi\rho^{2}d^{3/2}$ where $\xi$ is a constant associated with third-order property of $f$. Note $d^{3/2}$ is the compensation of normalizing $\rvs$.

Specifically, we try to prove $|\delta_{k}|\leq |\epsilon_{\rho}|$. We define a new function $g(x) = \frac{x}{AL_{0}+\sqrt{2}BL_{1}x}$, then to prove $|\delta_{k}|\leq |\epsilon_{\rho}|$ is equivalent to prove $|g(|\rvs^{T}\nabla f(\rvx_{k})|) - g(|\gamma_{k}|)| \leq d\big||\rvs^{T}\nabla f(\rvx_{k})| - |\gamma_{k}|\big|$, further it is equivalent to prove $g^{\prime}(x) = \frac{AL_{0}}{(AL_{0}+\sqrt{2}BL_{1}x)} \leq d$ when $x \geq 0$, which is obviously true. Overall, we have approximation error $|\delta_{k}| \leq \xi\rho^{2}d^{3/2}$.

Then, we continue our analysis. Suppose we do take step size $\alpha_{k} = \frac{|\gamma_{k}|}{(A L_{0} + \sqrt{2}BL_{1}|\gamma_{k}|)d}$ and substitute $\alpha_{k} = \alpha_{k}^{opt} + \delta_{k}$, then~\cref{append_l0l1_dy_1} can be re-formulate as 
\begin{align}
    \theta_{k+1} &\leq  \theta_{k} - (\alpha_{k}^{opt} + \delta_{k})|\rvs_{k}^{T}\nabla f(\rvx_{k})| + (\alpha_{k}^{opt} + \delta_{k})^{2}\frac{A L_{0} + \sqrt{2}BL_{1}|\rvs_{k}^{T}\nabla f(\rvx_{k})|}{2}d\nonumber\\
    & =  \theta_{k} - \frac{||\rvs^{T}\nabla f(\rvx_{k})||^{2}}{(A L_{0} + \sqrt{2}BL_{1}|\rvs^{T}\nabla f(\rvx_{k})|)d} \\
    &\;\;\;\;\;\;\;\;- |\rvs^{T}\nabla f(\rvx_{k})|\delta_{k} + \frac{||\rvs^{T}\nabla f(\rvx_{k})||^{2}}{2(A L_{0} + \sqrt{2}BL_{1}|\rvs^{T}\nabla f(\rvx_{k})|)d}\nonumber\\
    &\;\;\;\;\;\;\;\; + \frac{A L_{0} + \sqrt{2}BL_{1}|\rvs_{k}^{T}\nabla f(\rvx_{k})|}{2}d\delta_{k}^{2} + |\rvs^{T}\nabla f(\rvx_{k})|\delta_{k}\nonumber\\
    &\leq \theta_{k} - \frac{||\rvs^{T}\nabla f(\rvx_{k})||^{2}}{2(A L_{0} + \sqrt{2}BL_{1}|\rvs^{T}\nabla f(\rvx_{k})|)d} + \frac{(A L_{0} + \sqrt{2}BL_{1}|\rvs_{k}^{T}\nabla f(\rvx_{k})|)d}{2}\delta_{k}^{2}\nonumber\\
    &\leq \theta_{k} - \frac{||\nabla f(\rvx_{k})||^{2}}{4(A L_{0} + \sqrt{2}BL_{1}||\nabla f(\rvx_{k})||)d} + \frac{(A L_{0} + \sqrt{2}BL_{1}||\nabla f(\rvx_{k})||)d}{2}\delta_{k}^{2}\label[ineq]{append_l0l1_dynamic_00}\;\;\;\;\text{Lemma~\ref{append_proof_eq_expectation}}
\end{align}

\textbf{Condition 1}

Suppose $1 - \sqrt{2}BL_{1} \geq 0$ and $||\nabla f(\rvx_{k})|| \geq A L_{0} + \sqrt{2}BL_{1}||\nabla f(\rvx_{k})||$, ~\cref{append_l0l1_dynamic_00} can be reformulated as
\begin{align*}
    \theta_{k+1} &\leq \theta_{k} - \frac{||\nabla f(\rvx_{k})||}{4d} + \frac{||\nabla f(\rvx_{k})||d}{2}\delta_{k}^{2}
\end{align*}
Meanwhile, suppose $|\delta_{k}| \leq \xi\rho^{2}d^{3/2} \leq \frac{1}{2d}$, we have
\begin{align*}
     ||\nabla f(\rvx_{k})|| &\leq 8d(\theta_{k} - \theta_{k+1})\\
     \sum_{k=0}^{K} ||\nabla f(\rvx_{k})|| &\leq 8d(\theta_{0} - \theta_{k+1})
\end{align*}
We can conclude that there exists an iteration $j\sim [0, K]$ such that
\begin{align*}
    ||\nabla f(\rvx_{j})|| &\leq \frac{8d(\theta_{0} - \theta_{k+1})}{K}\\
    ||\nabla f(\rvx_{j})|| &\leq \frac{8d(f(\rvx_{0}) - f^{\star})}{K}
\end{align*}
which concludes that we need
\begin{align*}
    K \geq \frac{8d(f(\rvx_{0}) - f^{\star})}{\epsilon}
\end{align*}
iterations to reach $\epsilon$-first-order stationary point.    

\textbf{Condition 2}

Suppose $1 - \sqrt{2}BL_{1} \geq 0$ and $||\nabla f(\rvx_{k})|| \leq AL_{0} + \sqrt{2}BL_{1}||\nabla f(\rvx_{k})||$, we have $||\nabla f(\rvx_{k})||\leq \frac{AL_{0}}{1-\sqrt{2}BL_{1}}$.
Meanwhile, suppose $|\delta_{k}| \leq \xi\rho^{2}d^{3/2} \leq \frac{||\nabla f(\rvx_{k})||}{2(A L_{0} + \sqrt{2}BL_{1}||\nabla f(\rvx_{k})||)d}$, then~\cref{append_l0l1_dynamic_00} can be reformulated as
\begin{align*}
    \theta_{k+1} &\leq  \theta_{k} - \frac{||\nabla f(\rvx_{k})||^{2}}{8(A L_{0} + \sqrt{2}BL_{1}\frac{AL_{0}}{1-\sqrt{2}BL_{1}})d}\\
    ||\nabla f(\rvx_{k})||^{2} &\leq (\theta_{k} - \theta_{k+1})\frac{8AL_{0}d}{1 -\sqrt{2}BL_{1}}\\
    \sum_{k=0}^{K}||\nabla f(\rvx_{k})||^{2} &\leq (\theta_{0} - \theta_{k+1})\frac{8AL_{0}d}{1 -\sqrt{2}B L_{1}}
\end{align*}
We can conclude that there exists an iteration $j\sim [0, K]$ such that
\begin{align*}
    ||\nabla f(\rvx_{j})||^{2} &\leq \frac{8AL_{0}d(\theta_{0} - \theta_{k+1})}{(1 -\sqrt{2}BL_{1})K}\\
    ||\nabla f(\rvx_{j})|| &\leq \sqrt{\frac{8AL_{0}d(f(\rvx_{0}) - f^{\star})}{(1 -\sqrt{2}BL_{1})K}},
\end{align*}
which concludes that we need
\begin{align*}
    K \geq \frac{8AL_{0}d(f(\rvx_{0}) - f^{\star})}{(1 - \sqrt{2}BL_{1})\epsilon^{2}}
\end{align*}
iterations to reach $\epsilon$-first-order stationary point.

\textbf{Condition 3}

Suppose $1 - \sqrt{2}BL_{1} \leq 0$ and $||\nabla f(\rvx_{k})||^{2}\leq (\frac{AL_{0}}{1-\sqrt{2}BL_{1}})^{2}$.
Meanwhile, suppose $|\delta_{k}| \leq \xi\rho^{2}d^{3/2}  \leq \frac{||\nabla f(\rvx_{k})||}{2(A L_{0} + \sqrt{2}BL_{1}||\nabla f(\rvx_{k})||)d}$, then~\cref{append_l0l1_dynamic_00} can be reformulated as
\begin{align*}
    \theta_{k+1} &\leq  \theta_{k} - \frac{||\nabla f(\rvx_{k})||^{2}}{8(AL_{0} + \sqrt{2}BL_{1}|\frac{AL_{0}}{1-\sqrt{2}BL_{1}}|)d} \\
    ||\nabla f(\rvx_{k})||^{2} &\leq (\theta_{k} - \theta_{k+1})\frac{8AL_{0}d(2\sqrt{2}BL_{1} - 1)}{\sqrt{2}BL_{1} - 1}\\
    \sum_{k=0}^{K}||\nabla f(\rvx_{k})||^{2} &\leq (\theta_{0} - \theta_{k+1})\frac{8AL_{0}d(2\sqrt{2}BL_{1} - 1)}{\sqrt{2}BL_{1} - 1}
\end{align*}
We can conclude that there exists an iteration $j\sim [0, K]$ such that
\begin{align*}
    ||\nabla f(\rvx_{j})||^{2} &\leq \frac{8AL_{0}d(\theta_{0} - \theta_{k+1})(2\sqrt{2}BL_{1} - 1)}{(\sqrt{2}BL_{1} - 1)K}\\
    ||\nabla f(\rvx_{j})|| &\leq \sqrt{\frac{8AL_{0}d(f(\rvx_{0}) - f^{\star})(2\sqrt{2}BL_{1} - 1)}{(\sqrt{2}BL_{1} - 1)K}},
\end{align*}
which concludes that we need
\begin{align}
    K \geq \frac{8AL_{0}d(f(\rvx_{0}) - f^{\star})(2\sqrt{2}BL_{1} - 1)}{(\sqrt{2}BL_{1} - 1)\epsilon^{2}}
\end{align}
iterations to reach $\epsilon$-first-order stationary point.

\textbf{Condition 4}

Suppose $1 - \sqrt{2}BL_{1} \leq 0$ and $||\nabla f(\rvx_{k})||^{2}\geq (\frac{AL_{0}}{1-\sqrt{2}BL_{1}})^{2}$.
Meanwhile, suppose $\delta_{k} \leq \xi\rho^{2}d^{3/2} \leq \frac{||\nabla f(\rvx_{k})||}{2(A L_{0} + \sqrt{2}BL_{1}||\nabla f(\rvx_{k})||)d}$, then~\cref{append_l0l1_dynamic_00} can be reformulated as
\begin{align}
    \theta_{k+1} &\leq  \theta_{k} - \frac{(\frac{AL_{0}}{1-\sqrt{2}BL_{1}})^{2}}{8(A L_{0} + \sqrt{2}BL_{1}||\nabla f(\rvx_{k})||)d}\label[ineq]{append_l0l1_condition3_1}
\end{align}
Since $\frac{(\frac{AL_{0}}{1-\sqrt{2}BL_{1}})^{2}}{8(A L_{0} + \sqrt{2}BL_{1}||\nabla f(\rvx_{k})||)d}$ is a monotone decreasing function w.r.t. $||\nabla f(\rvx_{k})||$, then we can conclude that the loss function cannot be indicator of reaching $\epsilon$-first-order stationary points.
However, with an appropriate selection of parameters, the loss function can be minimized. I.e.,
\begin{align*}
      \theta_{k+1} &\leq  \theta_{k} - \frac{(\frac{AL_{0}}{\sqrt{2}BL_{1} - 1})^{2}}{8(A L_{0} + \sqrt{2}BL_{1}\frac{AL_{0}}{\sqrt{2}BL_{1} - 1})d}\\
      \theta_{k+1} &\leq  \theta_{k} - \frac{AL_{0}}{8(2\sqrt{2}BL_{1} - 1)(\sqrt{2}BL_{1} - 1)d}\\
      \theta_{k+1} &\leq  \theta_{0} - (K+1)\frac{AL_{0}}{8(2\sqrt{2}BL_{1} - 1)(\sqrt{2}BL_{1} - 1)d}\\
      \theta_{k} - f^{\star} &\leq  \theta_{0} - f^{\star} - K\frac{AL_{0}}{8(2\sqrt{2}BL_{1} - 1)(\sqrt{2}BL_{1} - 1)d},\\
\end{align*}
which concludes that we need
\begin{align*}
      K\geq \frac{8(2\sqrt{2}BL_{1} - 1)(\sqrt{2}BL_{1} - 1)(f(\rvx_{0}) - f^{\star} - \epsilon)d}{AL_{0}}
\end{align*}
iterations to reach ``local $\epsilon$-optimal point" (local minimum or saddle point).

We summarize the result overall conditions in Table~\ref{append_table_l0l1_summary}.

\bgroup
\setlength\tabcolsep{0.2em}
\begin{table*}[!ht]
\begin{center}
\small
\begin{threeparttable}
    \begin{tabular}{ccc}
    \toprule
         Conditions$^{[b]}$ & requirement over $\rho^{[a]}$ & Query complexity \\
        \cmidrule(r){1-3}
          $L_{1}\leq \frac{1}{\sqrt{2}B}$, $||\nabla f(\rvx)|| \geq \frac{AL_{0}}{1-\sqrt{2}BL_{1}}$ &$\rho\leq \frac{1}{d\sqrt{2\xi\sqrt{d}}}$& $\frac{8d(f(\rvx_{0}) - f^{\star})}{\epsilon}$\\
          $L_{1}\leq \frac{1}{\sqrt{2}B}$, $||\nabla f(\rvx)|| \leq \frac{AL_{0}}{1-\sqrt{2}BL_{1}}$ &$\rho \leq \frac{1}{d}\sqrt{\frac{\epsilon}{2\xi(A L_{0} + \sqrt{2}BL_{1}\epsilon)\sqrt{d}}}$&  $\frac{8AL_{0}d(f(\rvx_{0}) - f^{\star})}{(1 -\sqrt{2}B L_{1})\epsilon^{2}}$\\
           $L_{1}\geq \frac{1}{\sqrt{2}B}$, $||\nabla f(\rvx)|| \leq \frac{AL_{0}}{\sqrt{2}BL_{1} - 1}$ &$\rho \leq \frac{1}{d}\sqrt{\frac{\epsilon}{2\xi(A L_{0} + \sqrt{2}BL_{1}\epsilon)\sqrt{d}}}$& $\frac{8AL_{0}d(f(\rvx_{0}) - f^{\star})(2\sqrt{2}BL_{1} - 1)}{(\sqrt{2}BL_{1} - 1)\epsilon^{2}}$\\
           \cmidrule(r){1-3}
           $L_{1}\geq \frac{1}{\sqrt{2}B}$, $||\nabla f(\rvx)|| \geq \frac{AL_{0}}{\sqrt{2}BL_{1} - 1}$ &$\rho \leq \frac{1}{d}\sqrt{\frac{\epsilon}{2\xi(A L_{0} + \sqrt{2}BL_{1}\epsilon)\sqrt{d}}}$& $\frac{8(2\sqrt{2}BL_{1} - 1)(\sqrt{2}BL_{1} - 1)(f(\rvx_{0}) - f^{\star} - \epsilon)d}{AL_{0}}$ \\
    \bottomrule
    \end{tabular}
    \begin{tablenotes}
        \item $^{[a]}$ $\xi$ is a constant associated with third-order property of $f$, detailed in appendix~\cref{append_l0l1_dy_1}.
        \item $^{[b]}$ For the fourth condition, decreasing loss value instead of gradient norm, detailed in appendix~\cref{append_l0l1_condition3_1}.
    \end{tablenotes}
    \caption{With dynamic step size strategy, the convergence property of $f$ under relaxed smoothness.}
    \label{append_table_l0l1_summary}
\end{threeparttable}
\end{center}
\end{table*}
\egroup

Note that for the specific value of $A$ and $B$, we have $A = 1+e^{c}-\frac{e^{c}-1}{c}, B = \frac{e^{c} - 1}{c}$ and $||\rvx_{k+1} - \rvx_{k}|| = ||\alpha\rvs_{k}|| = \frac{\gamma_{k}}{(AL_{0} + \sqrt{2}BL_{1}\gamma_{k})\sqrt{d}}  \leq \frac{c}{L_{1}}\rightarrow c \geq \frac{1}{B\sqrt{2d}}\rightarrow e^{c} \geq 1 + \frac{1}{\sqrt{2d}}$.
It is easy to see that such $c$ exists, we can safely consider $A = 1.01, B= 1.01$ for simplicity (under large $d$) since $A$ and $B$ are expected to be small values.
\end{proof}

\clearpage
\section{Experiments}
\label{append_experiments}

\subsection{Setup}
\label{append_exp_setup}
For the image tasks, we use the same batch size of 1024, training epoch of 200, and input image sizes 32$\times$32 for MNIST and CIAFR10 datasets. 
For the language tasks, we use the same batch size of 32, training epoch of 100 and 500 for $\{$SST-2, QNLI$\}$ and $\{$MRPC, STS-B, RTE$\}$ respectively. And, we select the MRPC task, which has a relatively small dataset, as our base task to search for the best learning rate. Please refer to Appendix~\ref{append_MRPC_lr_selection} for the learning rate selection.

\textbf{Computer resources} All the experiments can be run on a single NVIDIA RTX A5000 Graphics Card (24G). Each image task in one setting can be completed within 2 hours, using less than 5GB of GPU memory. Similarly, each language task in one setting can be completed within 75 hours, using less than 3GB of GPU memory.

\clearpage
\subsection{Performance of image tasks with 200 training epochs under various learning rates}
\label{append_perf_various_lr}

\begin{figure*}[!ht]
\begin{center}
\subfigure[MeZO with various learning rates under MNIST.]
{
\includegraphics[width=1.0\linewidth]{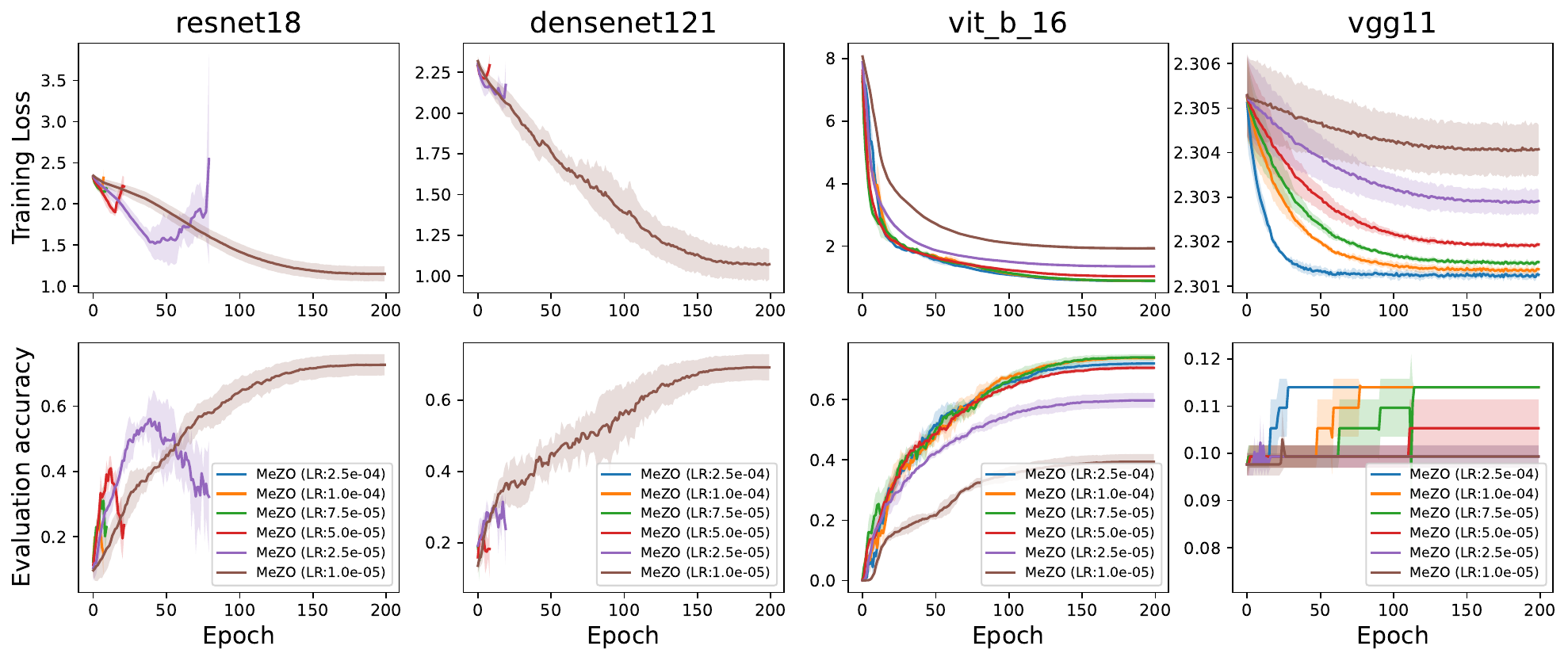}
}
\subfigure[STP with various learning rates under MNIST.]
{
\includegraphics[width=1.0\linewidth]{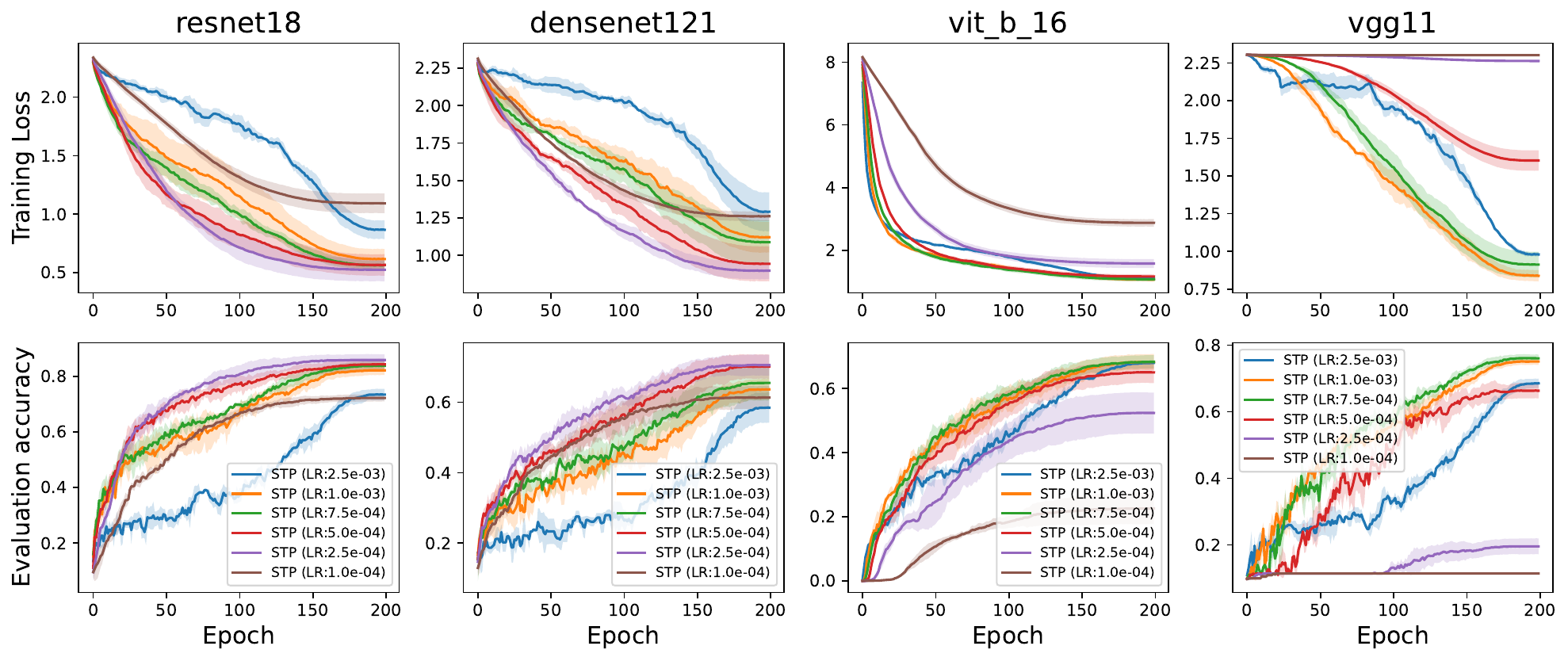}
}
\subfigure[VS2P with various learning rates under MNIST.]
{
\includegraphics[width=1.0\linewidth]{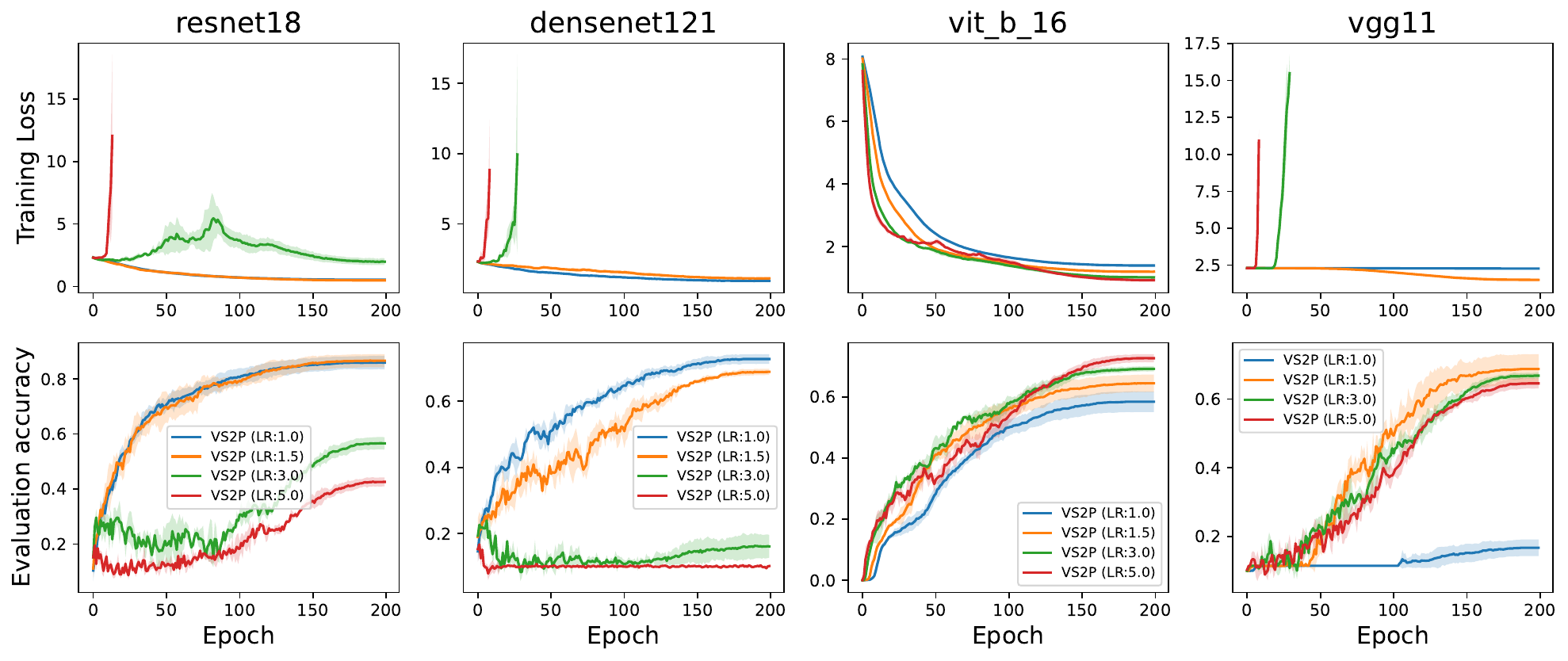}
}
\caption{Fully fine-tuning DenseNet121, ResNet18, ViT-B16, and VGG11 under various learning rates \& datasets MNIST \& 200 training epochs \& 3 different seeds.}
\end{center}
\end{figure*}

\begin{figure*}[!ht]
\begin{center}
\subfigure[MeZO with various learning rates under CIFAR10.]
{
\includegraphics[width=1.0\linewidth]{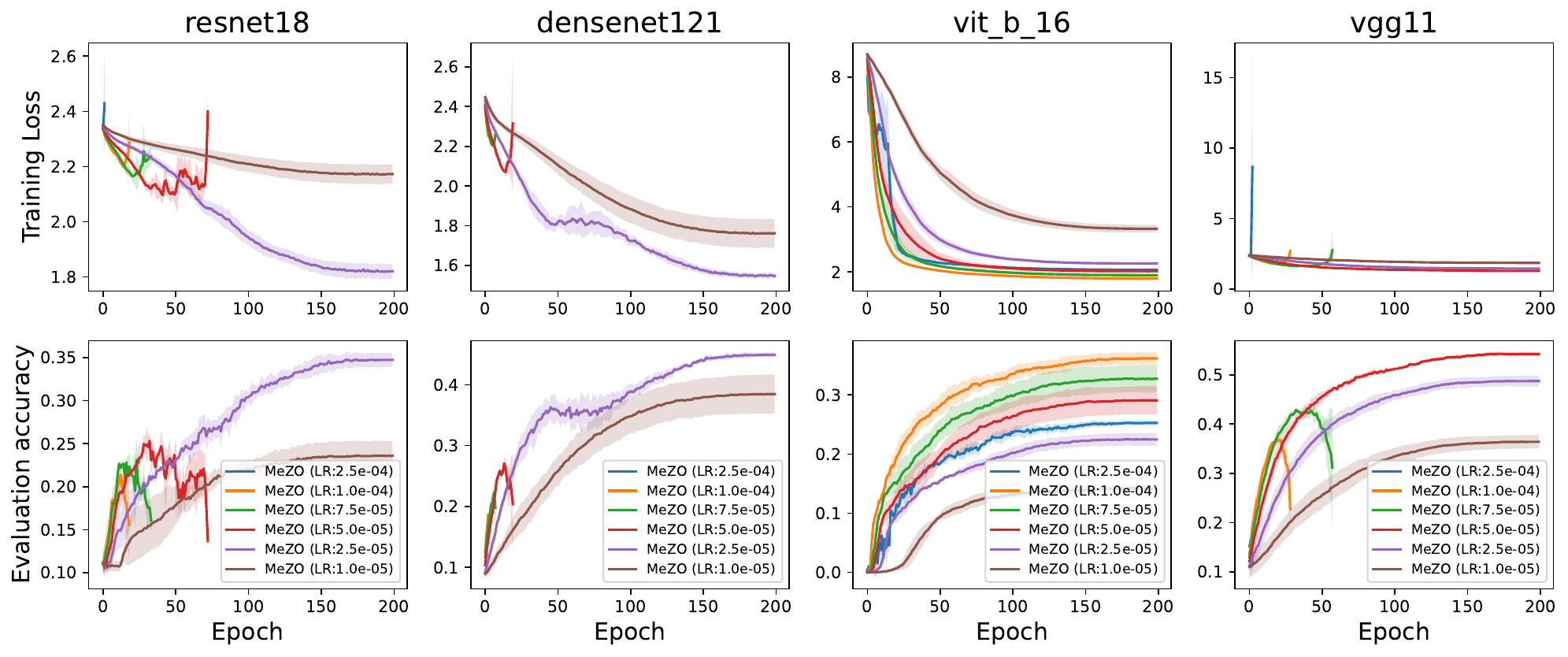}
}
\subfigure[STP with various learning rates under CIFAR10.]
{
\includegraphics[width=1.0\linewidth]{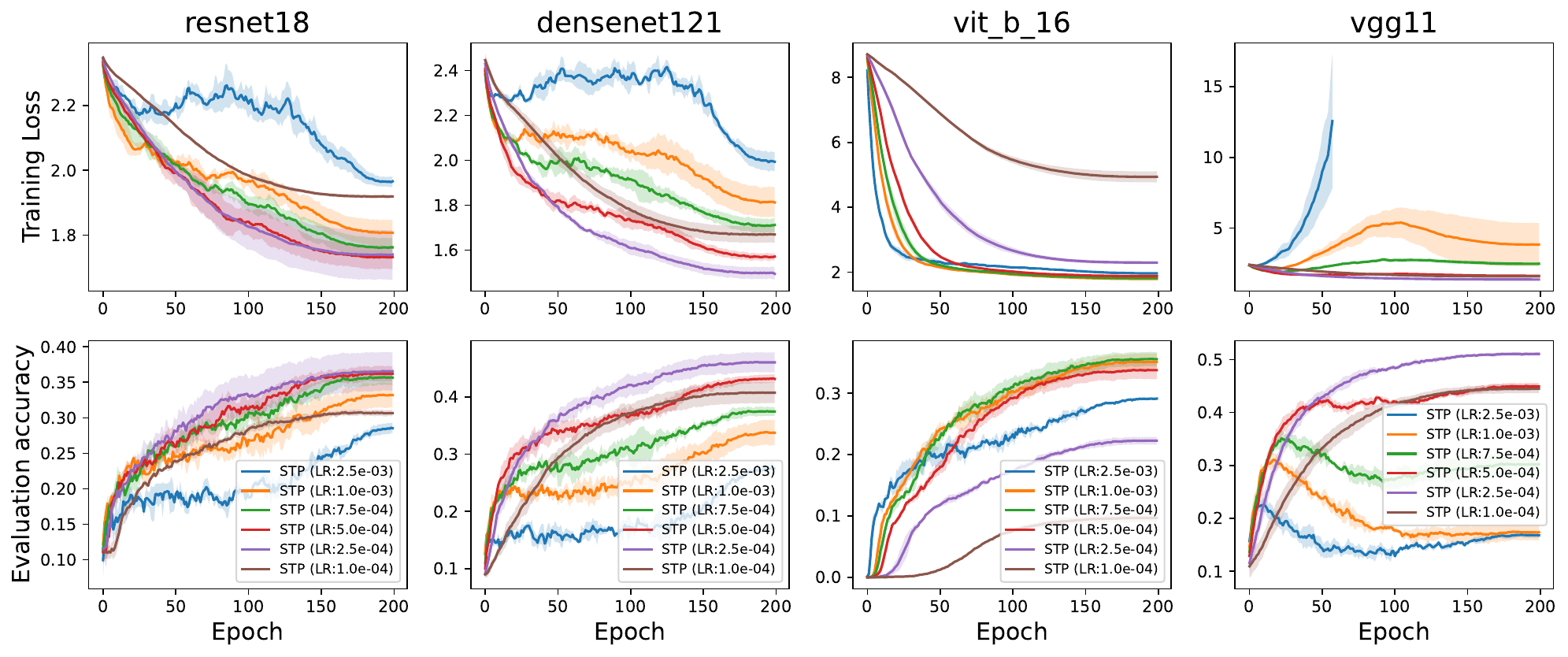}
}
\subfigure[VS2P with various learning rates under CIFAR10.]
{
\includegraphics[width=1.0\linewidth]{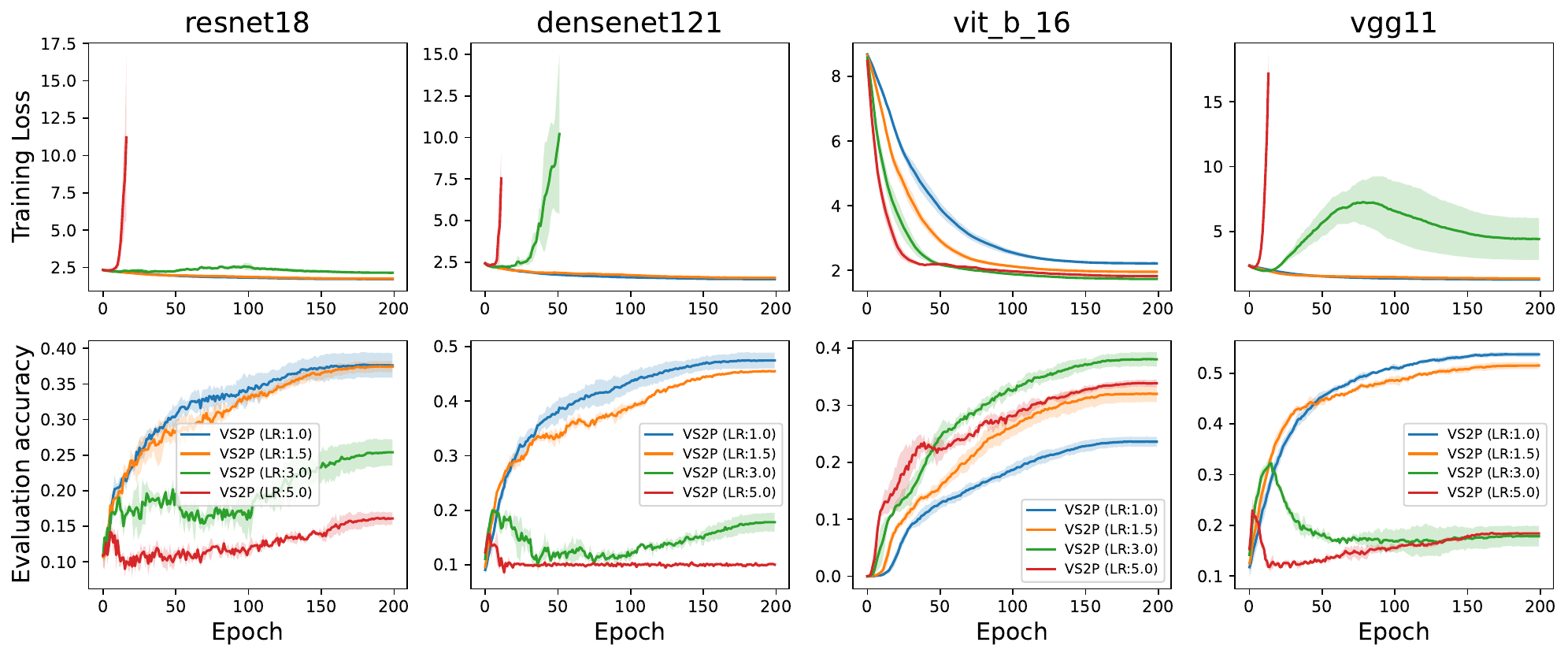}
}
\caption{Fully fine-tuning DenseNet121, ResNet18, ViT-B16, and VGG11 under various learning rates \& datasets CIFAR10 \& 200 training epochs \& 3 different seeds.}
\end{center}
\end{figure*}

\clearpage
\subsection{Performance of image tasks with 100 training epochs under best learning rate}
\label{append_perf_img_task_100_epochs}

The dynamics of the training process including the training loss curves and evaluation accuracy along with varying epochs are summarized in the following figure. 
Accordingly, we summarize the performance of different methods under various learning rates in Appendix~\ref{append_perf_various_lr_100_epoch}.

\begin{figure*}[!ht]
\label{append_fig_img_convergence_and_test_acc_100_epochs}
\begin{center}
\subfigure[Under datasets MNIST.]
{
\includegraphics[width=1.0\linewidth]{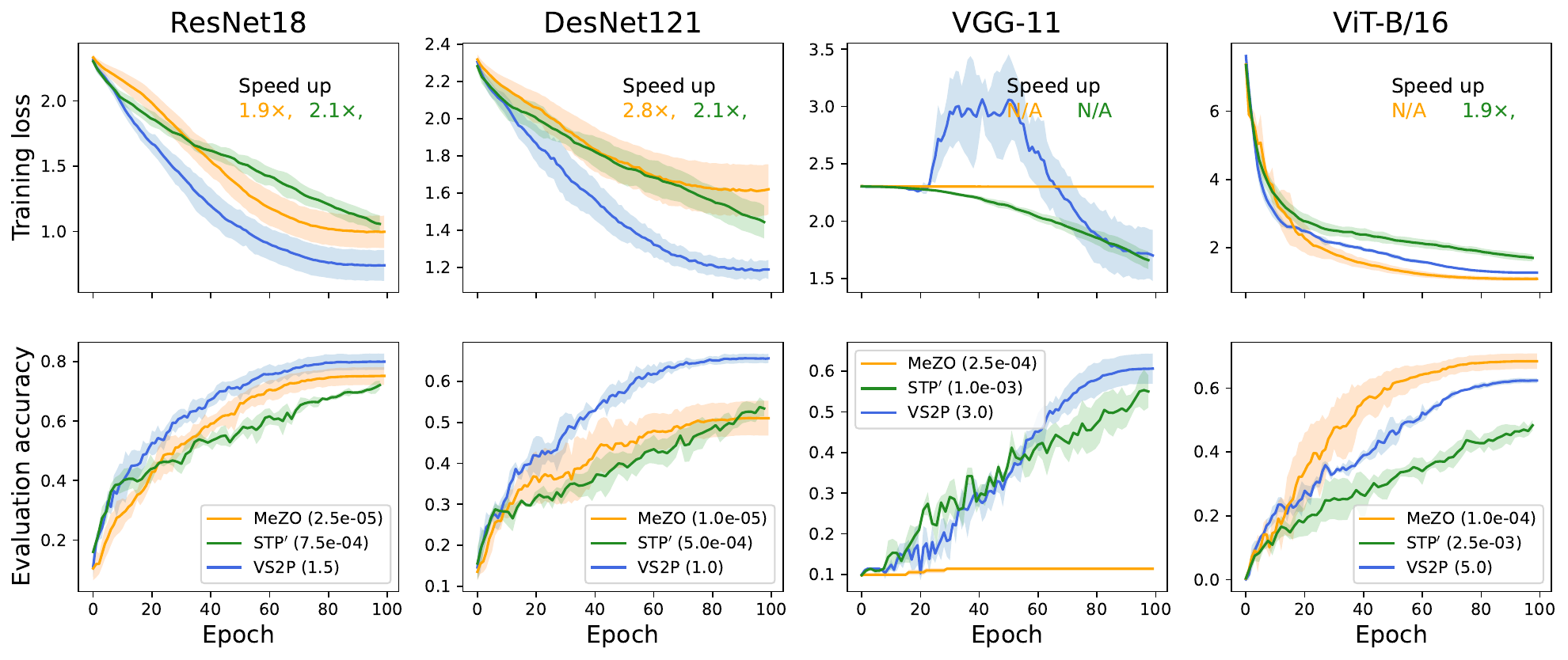}
}
\subfigure[Under datasets CIFAR10.]
{
\includegraphics[width=1.0\linewidth]{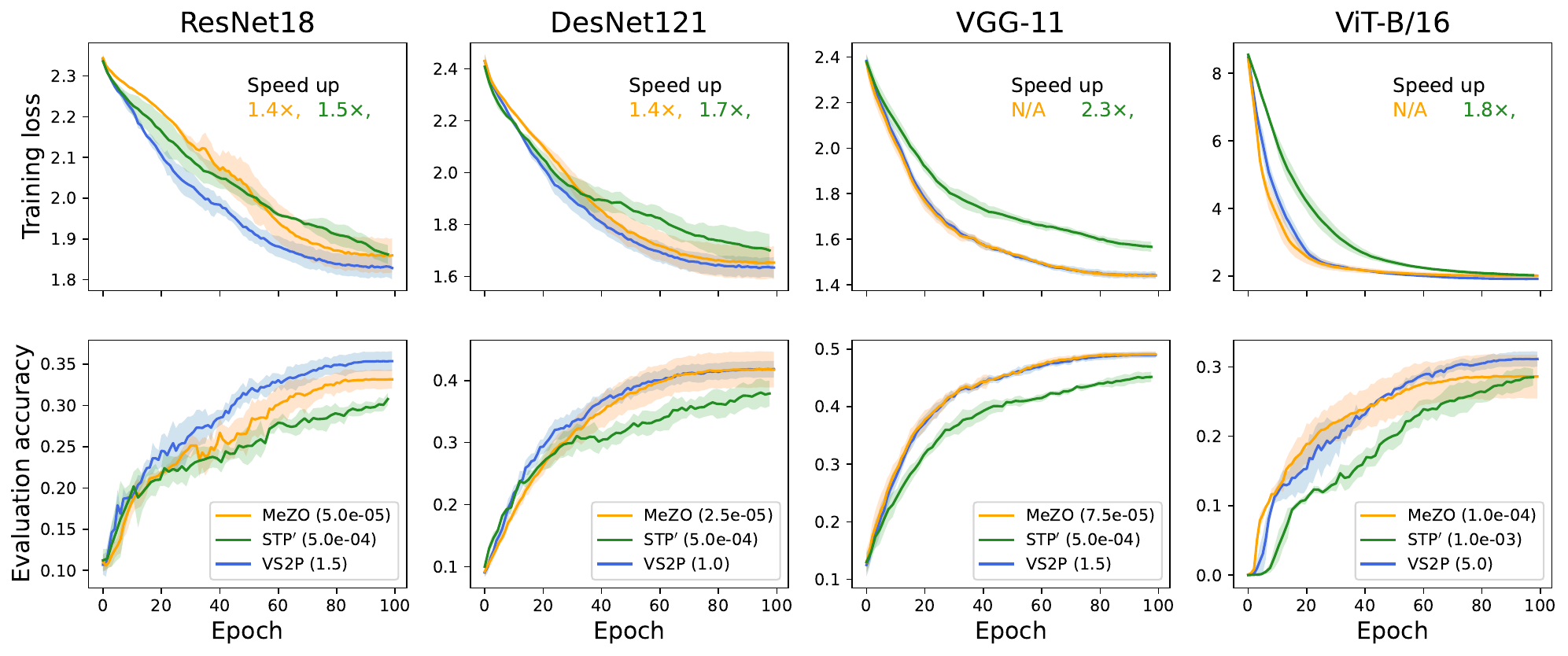}
}
\caption{Fully fine-tuning DenseNet121, ResNet18, ViT-B16, and VGG11 with proposed method and baselines under datasets MNIST and CIFAR10 \& 100 epochs \& 3 different seeds.}
\end{center}
\end{figure*}

\subsection{Performance of image tasks with 100 training epochs under various learning rates}
\label{append_perf_various_lr_100_epoch}

\begin{figure*}[!ht]
\begin{center}
\subfigure[MeZO with various learning rates under MNIST.]
{
\includegraphics[width=1.0\linewidth]{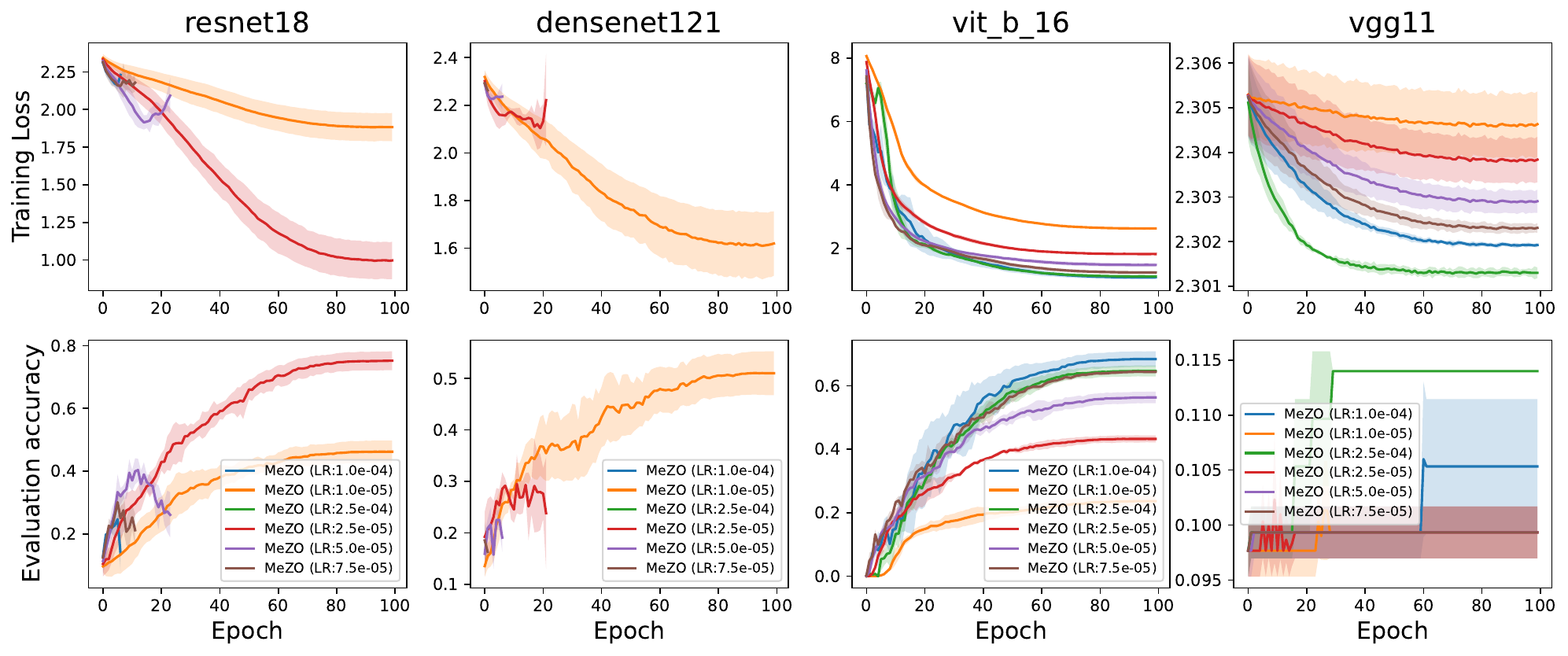}
}
\subfigure[STP with various learning rates under MNIST.]
{
\includegraphics[width=1.0\linewidth]{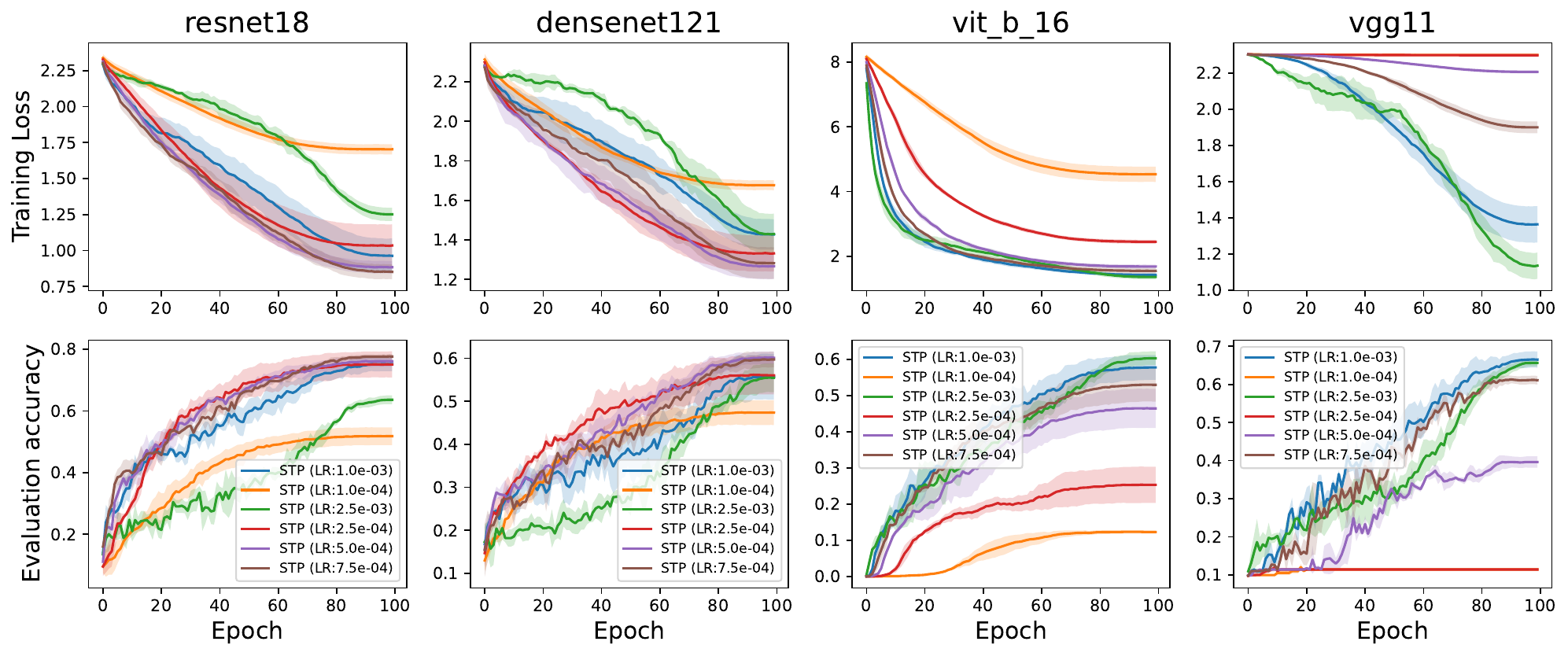}
}
\subfigure[VS2P with various learning rates under MNIST.]
{
\includegraphics[width=1.0\linewidth]{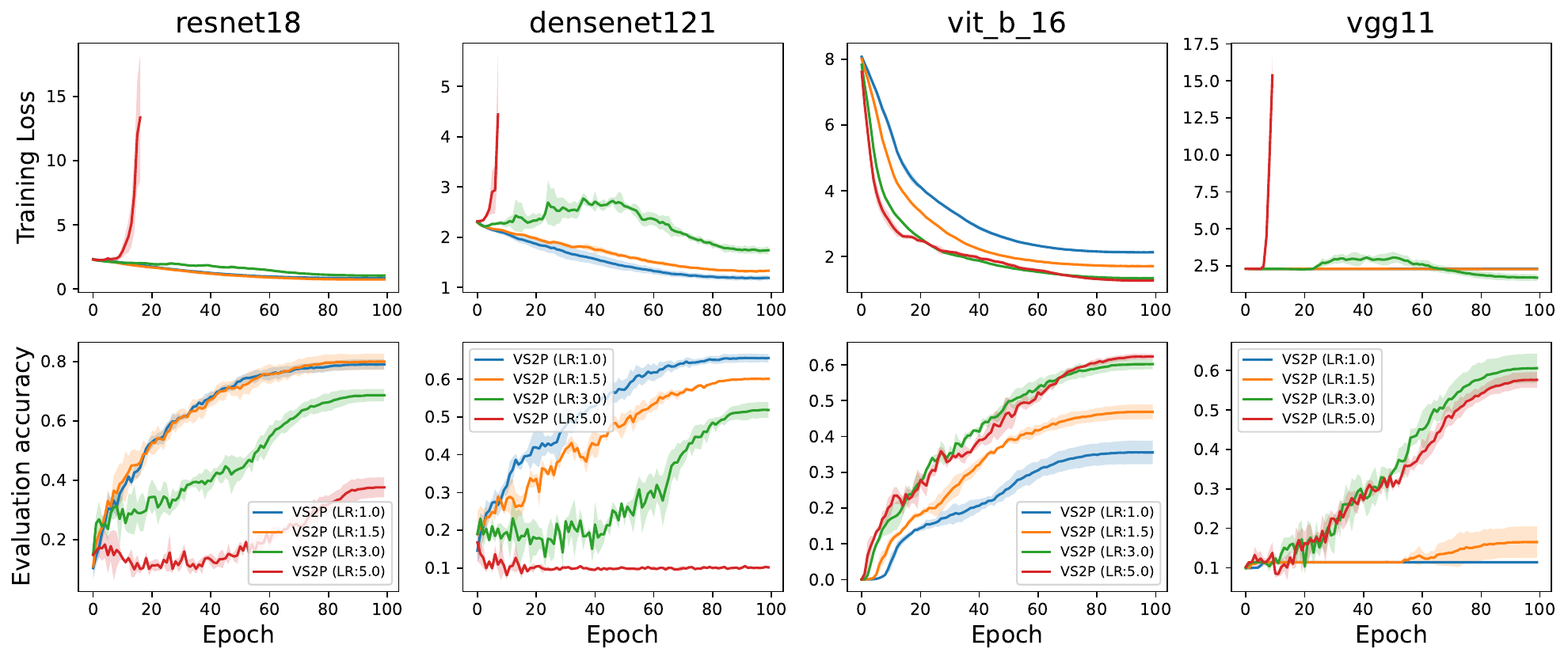}
}
\caption{Fully fine-tuning DenseNet121, ResNet18, ViT-B16, and VGG11 under various learning rates \& datasets MNIST \& 100 training epochs \& 3 different seeds.}
\end{center}
\end{figure*}

\begin{figure*}[!ht]
\begin{center}
\subfigure[MeZO with various learning rates under CIFAR10.]
{
\includegraphics[width=1.0\linewidth]{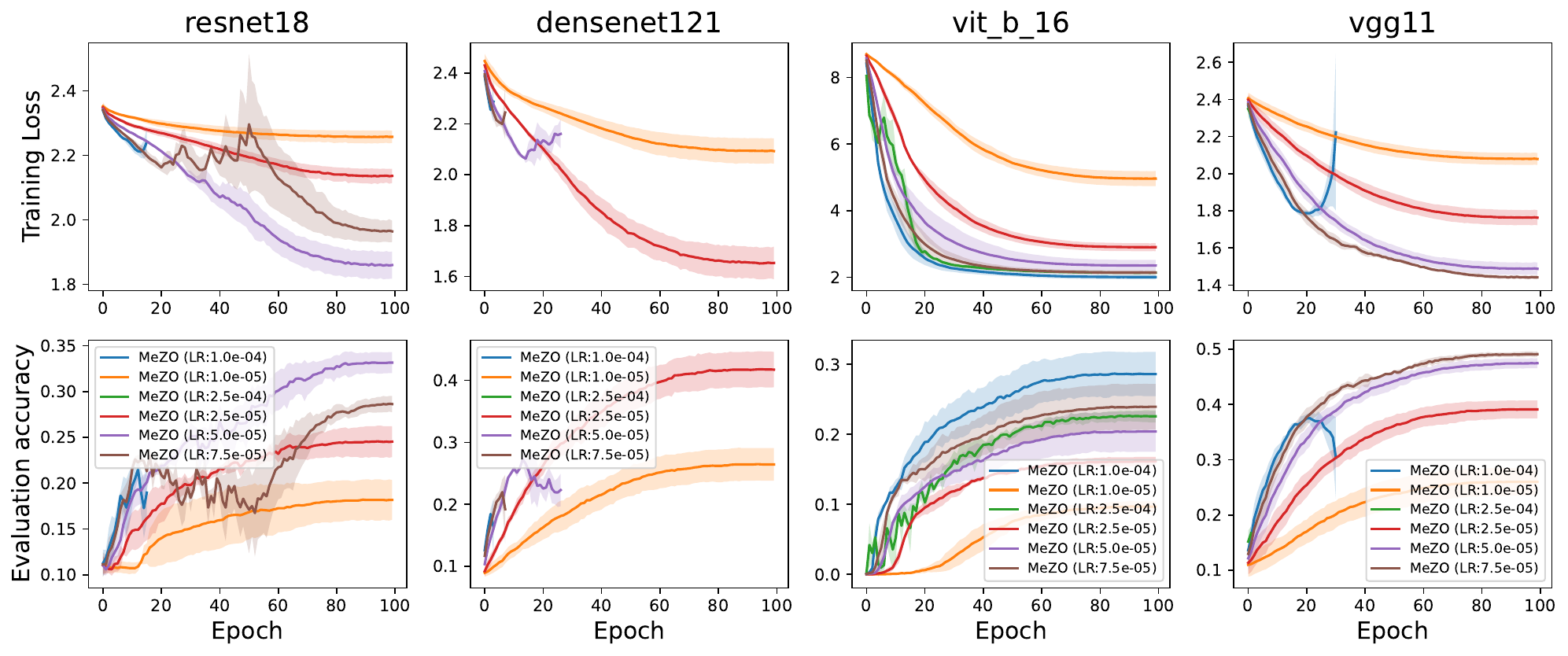}
}
\subfigure[STP with various learning rates under CIFAR10.]
{
\includegraphics[width=1.0\linewidth]{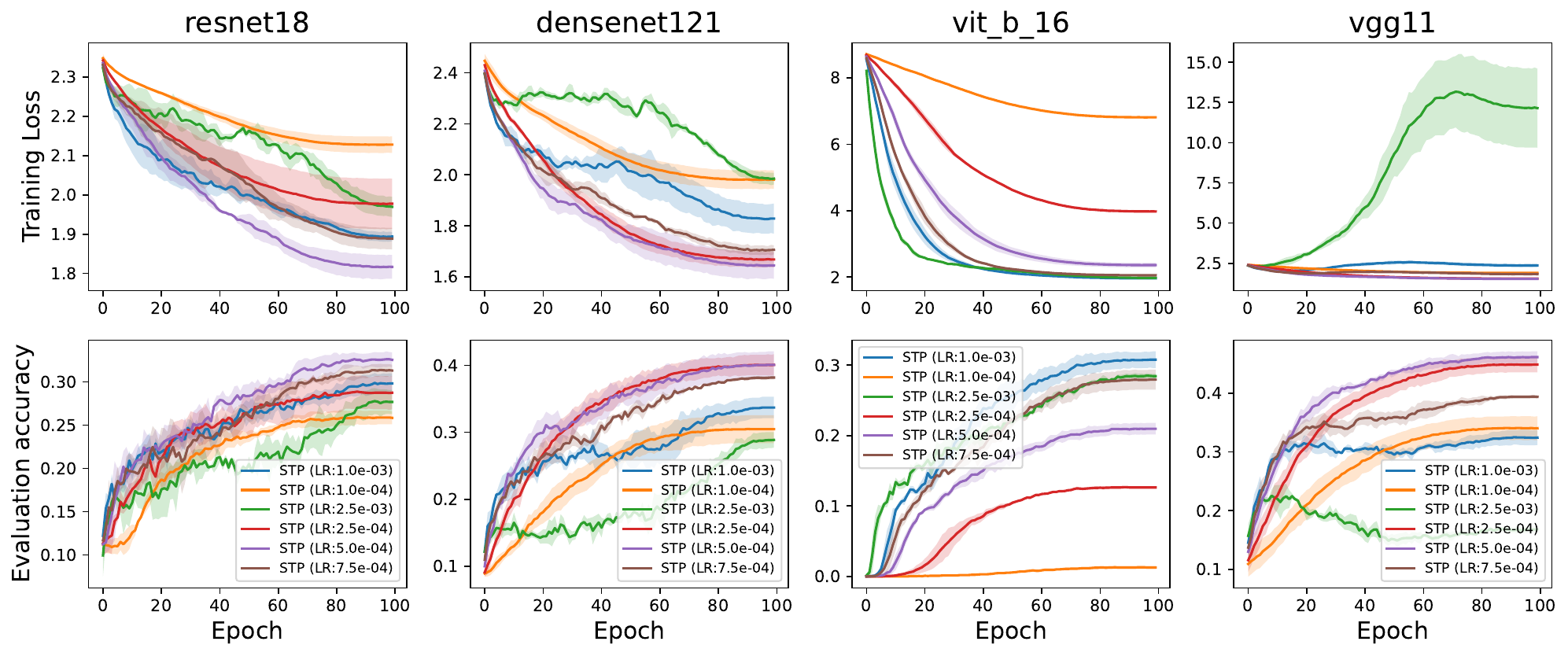}
}
\subfigure[VS2P with various learning rates under CIFAR10.]
{
\includegraphics[width=1.0\linewidth]{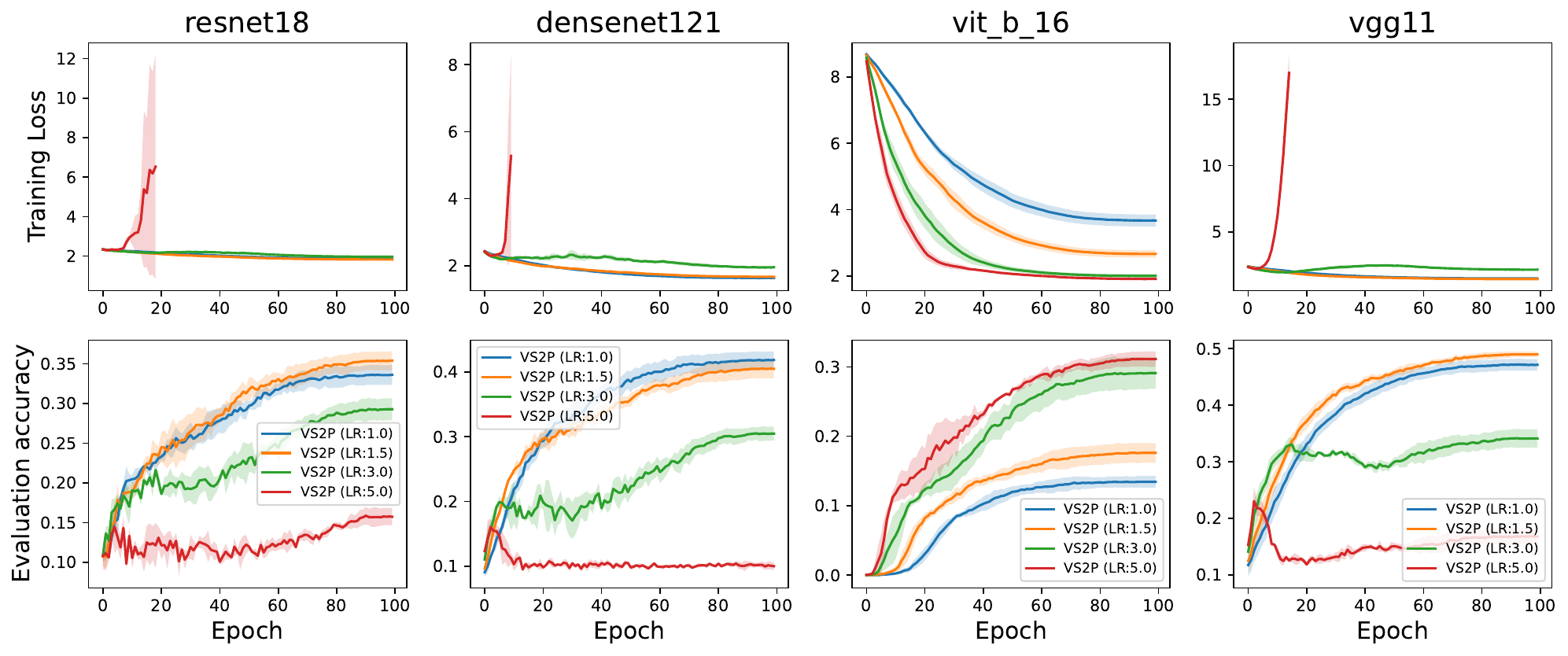}
}
\caption{Fully fine-tuning DenseNet121, ResNet18, ViT-B16, and VGG11 under various learning rates \& datasets CIFAR10 \& 100 training epochs \& 3 different seeds.}
\end{center}
\end{figure*}

\clearpage
\subsection{Performance of MRPC tasks under various learning rates}
\label{append_MRPC_lr_selection}

\begin{figure*}[!ht]
\begin{center}
\subfigure{\includegraphics[width=1.0\linewidth]{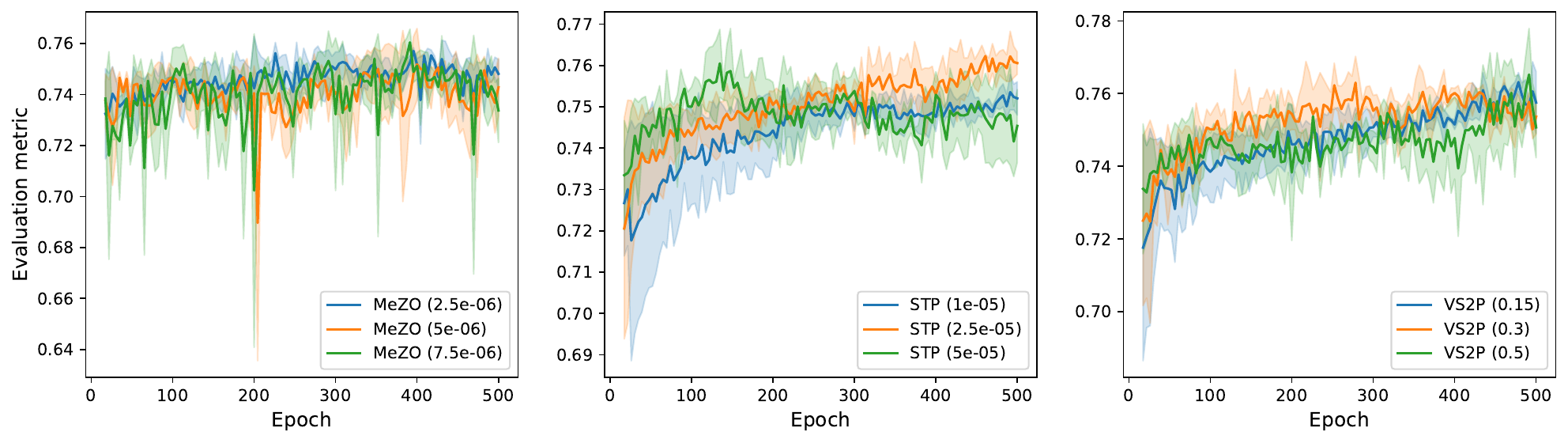}}
\caption{Performance of MRPC tasks employing MeZO, STP, and proposed VS2P methods under various learning rates.}
\end{center}
\end{figure*}

\subsection{Failure case study}
\label{append_failure_study}

\begin{figure*}[!ht]
\begin{center}
\subfigure{\includegraphics[width=1.0\linewidth]{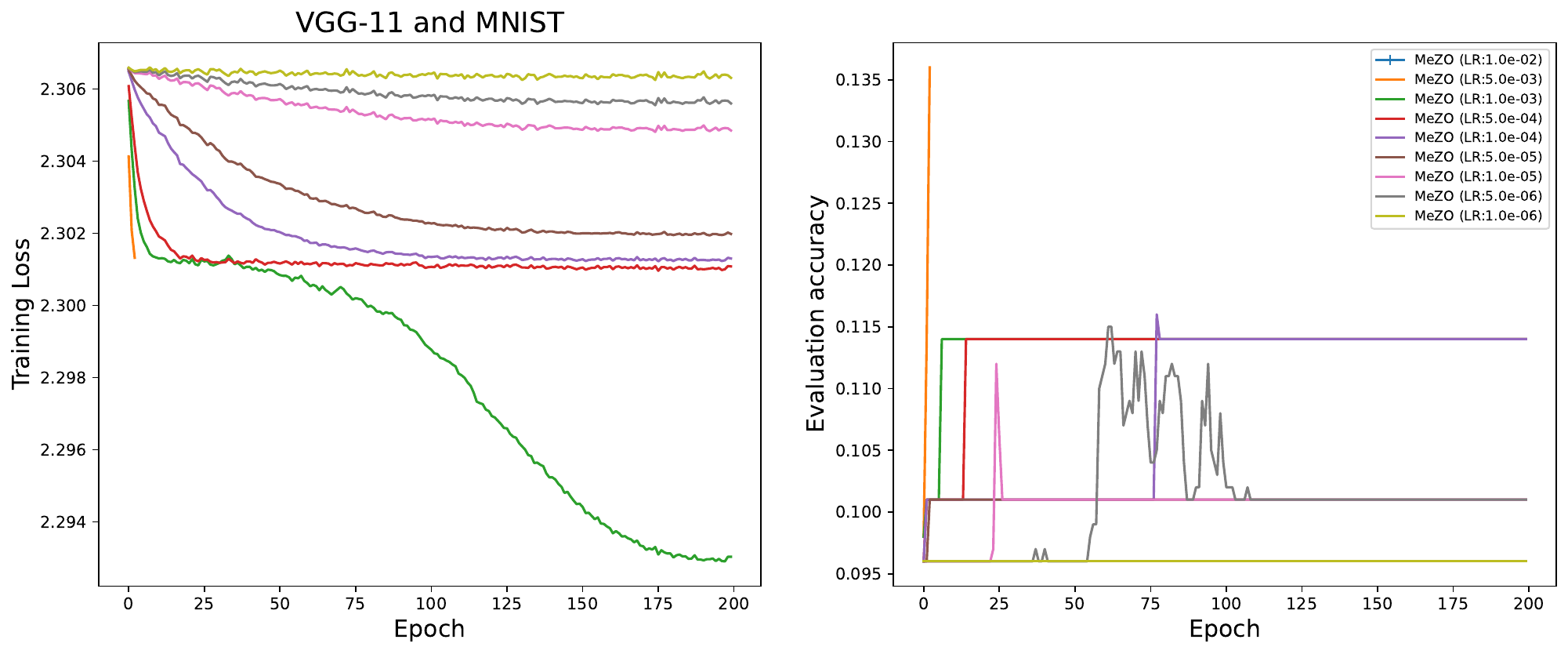}}
\caption{Failure case study. Performance of MeZO method with setting VGG-11 and MNIST dataset under various learning rates.}
\end{center}
\end{figure*}

\section*{Broader impact}
Our work primarily focuses on theoretical and practical developments in zeroth-order optimization methods, which will enable memory-efficient model training of deep model optimization tasks. 
However, we are also aware that the advancements may have broader implications, some of which could potentially have negative social impacts, such as misuse of the method in malicious application developments.

\end{document}